\documentclass[preprint,12pt,authoryear]{elsarticle}

\usepackage{amssymb}
\usepackage{amsmath}
\usepackage{amsthm}

\usepackage[percent]{overpic}
\usepackage{algorithmic}
\usepackage{algorithm}
\usepackage{booktabs}

\usepackage{xcolor}
\usepackage[normalem]{ulem}

\newtheorem{definition}{Definition}[section]
\newtheorem{theorem}{Theorem}[section]
\newtheorem{lemma}{Lemma}[section]
\newtheorem{remark}{Remark}

\newcommand*{\dif}{\mathop{}\!\mathrm{d}}

\journal{Neural Networks}

\begin{document}

\begin{frontmatter}

\title{Neural Networks Trained by Weight Permutation are Universal Approximators\tnoteref{label1}} 
\tnotetext[label1]{Published in \emph{Neural Networks}, Volume 187, 2025, 10.1016/j.neunet.2025.107277.}
 
\author[BNU]{Yongqiang Cai}
\ead{caiyq.math@bnu.edu.cn}
\affiliation[BNU]{organization={School of Mathematical Sciences, Laboratory of Mathematics and Complex Systems, MOE, Beijing Normal University},
           city={Beijing},
           postcode={100875}, 
           country={China}}

\author[PolyU1]{Gaohang Chen\corref{cor1}}
\ead{gaohang.chen@connect.polyu.hk}
\cortext[cor1]{Corresponding author.}
\affiliation[PolyU1]{organization={Department of Applied Mathematics, The Hong Kong Polytechnic University},
           addressline={Hung Hom, Kowloon}, 
           country={Hong Kong}
           }

\author[PolyU2]{Zhonghua Qiao}
\ead{zqiao@polyu.edu.hk}
\affiliation[PolyU2]{organization={Department of Applied Mathematics \& Research Institute for Smart Energy, The Hong Kong Polytechnic University},
           addressline={Hung Hom, Kowloon}, 
           country={Hong Kong}
           }

\begin{abstract}
The universal approximation property is fundamental to the success of neural networks, and has traditionally been achieved by training networks without any constraints on their parameters. However, recent experimental research proposed a novel permutation-based training method, which exhibited a desired classification performance without modifying the exact weight values. In this paper, we provide a theoretical guarantee of this permutation training method by proving its ability to guide a ReLU network to approximate one-dimensional continuous functions. Our numerical results further validate this method's efficiency in regression tasks with various initializations. The notable observations during weight permutation suggest that permutation training can provide an innovative tool for describing network learning behavior.
\end{abstract}


\end{frontmatter}


\section{Introduction}

The Universal Approximation Property (UAP) of neural networks is a cornerstone in the theoretical guarantee of deep learning, proving that even the simplest two-layer feedforward networks can approximate any continuous function \citep{cybenko1989approximation, Hornik1989Multilayer, Leshno1993Multilayer}.
This fascinating ability allows neural networks to replace critical, challenging components in existing frameworks to enhance efficiency
\citep{raissi2019physics, lu2021learning}.
Despite the extensive study in various settings \citep{castro2000neural, fan2020universal, wang2022approximation}, existing UAP research rarely imposes restrictions on the network parameters.
However, in specific application cases, constraints are essential to meet certain requirements \citep{nugent2005physical, kosuge202116}.

As a constrained scenario, \citet{qiu2020train} empirically showed that, without altering the exact value of network weights,
only permuting the initialized weights can achieve comparable or better performance for image classification tasks.
This unique property makes the permutation-based training method attractive for specific hardware applications, and has been utilized as fixed-weight accelerators \citep{kosuge202116}. It can also facilitate the implementation
of physical neural networks \citep{nugent2005physical, feldmann2021parallel}.
Despite its impressive benefits in applications, the permutation training method suffers from an absence of theoretical guarantees regarding its effectiveness, which hinders further algorithmic and hardware development necessary to fully exploit its potential.

This paper establishes the first theoretical foundation of this method (to our best knowledge) by proving the UAP of a permutation-trained Rectified Linear Unit (ReLU) network with random initializations for any one-dimensional continuous function.
Compared to the conventional UAP scenarios, the proof of permutation training encounters a significantly greater challenge,
primarily due to the extreme constraints of maintaining the initialized weight values.
The key proof idea is a four-pair construction of the step function approximators, which enables the approximation through a piecewise constant function \citep{stein2009real}.
Additionally, a reorganization method is proposed to eliminate the impact of the remaining weights.

Our numerical experiments not only validate our theoretical results by illustrating the widespread existence of the UAP of permutation training in diverse initializations, but also emphasize the effects of initializations on the permutation training performance.
Moreover, the patterns observed during permutation training also highlight its potential in describing learning behavior, relating to topics like the pruning technique \citep{frankle2019lottery} and continual learning \citep{maltoni2019continuous, zeng2019continual}.
Our main findings are summarized below:
\begin{enumerate}
\item We prove the UAP of permutation-trained ReLU networks with pairwise random initialization to one-dimensional continuous functions.
\item The numerical experiments of regression problems emphasize the crucial role played by the initializations in the permutation training scenario.
\item By observing the permutation patterns, we find that permutation training as a new approach holds promise in describing intricate learning behaviors.
\end{enumerate}

\subsection{Permutation training's advantages in hardware implementation}

Permutation training has been applied to training large models on extensive databases for image classification tasks \citep{qiu2020train}, achieving performance that is comparable or even superior to conventional free-training training methods like Adam \citep{kingma2015adam}.
However, there is currently no evidence to report a significant advantage when applied to more diverse tasks on contemporary Graphics Processing Unit (GPU)-based hardware.
Nevertheless, we believe it is highly suitable for the design of \emph{physical neural networks} \citep{nugent2005physical}, since reconnecting the neurons is sometimes more convenient than altering the exact value.
Therefore, permutation training may inspire alternative physical weight connection implementations, such as using fixed-weight devices controlled by a permutation circuit \citep{qiu2020train}. This idea has been applied to a fixed-weight network accelerator \citep{kosuge202116, kosuge20210}.

Another potential application scenario is the physical neural networks with an explicit structure to store the weight value, such as the \emph{integrated photonic tensor core}, a computing chip with specialized architecture \citep{feldmann2021parallel}. This design has been successfully employed by international commercial companies
in their photonic computing products. In each photonic tensor core, an array of phase-change cells are organized to separately store each element of the weight matrix, with their values adjusted through optical modulation of the transmission states of the cells. However, permutation training indicates that, in addition to changing the exact value, it is feasible to connect each cell with the permutation circuit for convenient reconnections. Therefore permutation training can facilitate the learning process.

\subsection{Related works}
The UAP has been extensively studied in various settings, leading to many efficient applications. It is well known that fully connected networks are universal approximators for continuous functions \citep{cybenko1989approximation, Hornik1989Multilayer, Leshno1993Multilayer}.
Further research primarily concentrated on how the width and depth of models influence their UAP.
An estimation has been made of the lower bound for the minimum width required by networks to achieve UAP \citep{cai2023achieve, Lu2017Expressive}. It has also been demonstrated that the depth of the network is crucial, as increasing the network depth can enhance the expression power \citep{shen2022optimal, telgarsky2016benefits,yarotsky2017error}.
Our work considers the network with arbitrary width, aligning with conventional scenarios.
However, the constraints on parameters imposed by permutation training introduce non-trivial complexities and challenges not encountered in traditional free training methods.

Permutation, as a typical group structure, has been systematically described \citep{cameron1999permutation}. In deep learning, it closely relates to permutation equivariant or invariant networks \citep{cohen2016group} designed to learn from symmetrical data \citep{Lee2019Set, Zaheer2017Deep}.
It is also evident in graph-structured data which inherently exhibit permutation invariance \citep{maron2018invariant, satorras2021n}.
However, permutation training is not limited to the issues with intrinsic symmetry.

As for the weight permutation attempts, \citet{qiu2020train} empirically proposed the first (to our knowledge) weight-permuted training method. This method preserves the initialized weight value, allowing more efficient and reconfigurable implementation of the physical neural networks \citep{kosuge202116, kosuge20210}.
Our work provides theoretical guarantees of this method and considers some regression tasks numerically. Additionally, initialization can be improved by rewiring neurons from the perspective of computer networks \citep{scabini2022improving}, but the training methods are unchanged.

Permutation training is also closely related to the permutation symmetry and Linear Mode Connectivity (LMC) \citep{entezari2021role,frankle2020linear}. The LMC suggests that after a proper permutation, most stochastic gradient descent solutions under different initialization will fall in the same basin in the loss landscape.
Similarly, permutation training also seeks a permutation to improve performance. Therefore, the search algorithms utilized in LMC indicate the possibility of more efficient permutation training algorithms, as the fastest algorithm can search a proper permutation of large ResNet models in seconds to minutes \citep{ainsworth2023git,jordan2023repair}.

\subsection{Outline}
We state the main result and proof idea in Section~\ref{sec:main}.
In Section~\ref{sec:proof}, we provide a detailed construction of the proof.
The numerical results of permutation training are presented in Section~\ref{sec:experiments}, along with the observation of permutation behavior during the training process. In Section~\ref{sec:discussion} we discuss the possible future work.
Finally, the conclusion is provided in Section~\ref{sec:conclusion}.

\section{Notations and main results}
\label{sec:main}

This section introduces the notations and UAP theorems, accompanied by a brief discussion of the proof idea.

\subsection{Nerual networks architecture}

We start with a one-hidden-layer feed-forward ReLU network with $N$ hidden neurons. It has the form of a linear combination of ReLU basis functions like
\begin{equation*}
    f(x) = \sum_{i=1}^N a_i \text{ReLU}(w_i x + b_i)  + c, \quad \text{ReLU}(z) = \max\{z, 0\},
\end{equation*}
where all parameters are scalars when approximating one-dimensional functions,
 \emph{i.e.}, $w_i, b_i, a_i, c \in \mathbb{R}$.
Since ReLU activation is positively homogeneous \emph{i.e.}, $\text{ReLU}(\lambda x) = \lambda \text{ReLU}(x)$ for all $\lambda>0$, we consider a homogeneous case with $w_i = \pm 1$. To facilitate our construction below, we assume $N$ is an even number (\emph{i.e.}, $N = 2n$) and the basis functions located pairwisely as
\begin{equation} \label{basis}
    \phi_k^\pm (x) = \text{ReLU}\big(\pm (x - b_k) \big), \quad
    k = 1,2,...,n,
\end{equation}
where the biases $\{b_k\}_{k = 1}^n$ determine the basis locations \footnote{In this paper, we use this notation $\{a_k\}_{k = 1}^n := \{a_1, \cdots, a_n\}$ to represent a set with $n$ elements, while $(a_k)_{k = 1}^n := (a_1, \cdots, a_n)$ for a vector with $n$ components.}.
The requirement of even $N$ can be removed by adding a "ghost basis function" $\phi_k^\pm (x)$ with $a_i = 0$ if it can be paired with an existing $\phi_k^\mp (x)$ also with $a_i = 0$. Nevertheless, we retain this condition for simplicity.
Next, we introduce two factors $\alpha, \gamma$ to adjust the network's output, leading to an additional one-dimensional linear layer. While this layer is not essential for achieving UAP, it does simplify the proof and offer practical value.
The network's output function $f^{\text{NN}}$ gives
\begin{equation} \label{NN}
    f^{\text{NN}}(x) = \alpha + \gamma \sum_{k=1}^n \big[ p_k \phi_k^+ (x) + q_k \phi_k^- (x) \big],
\end{equation}
where $\theta^{(2n)} = (p_1,q_1,...,p_n,q_n) \in [-1, 1]^{2n}$ are the coefficient vector of basis functions,
which also correspond to the parameters in the second hidden layer of the network.

\subsection{Permutation and corresponding properties}

The permutation of a vector can be described as a process of rearranging the elements within, leaving the actual value unchanged. Concretely, it can be defined as

\begin{definition}
    For a vector $V^{(m)} = (v_1, v_2, \cdots, v_{m})$, the permutation $\tau$ is a bijection from the element set $\{ v_i \}_{i = 1}^m$ to itself.
\end{definition}
We denote the permuted vector as $\tau (V^{(m)}) = (\tau(v_i))_{i = 1}^m$. Notice that all permutations of $V^{(m)}$ can form a group $S_m$, which is closed under composition. Precisely, for any $\pi, \tau \in S_m$, there is a $\rho \in S_m$ such that $\rho(v_i) = \pi ( \tau(v_i) )$ for all $v_i$ within $V^{(m)}$.
This property leads to the major advance of permutation training: ideally, the final weight $\theta_T$ can be achieved by permuting the initialized weight $\theta_0$ only once, \emph{i.e.}, there exists a $\tau \in S_{2n}$ such that $\theta_T = \tau(\theta_0)$. This \emph{one-step} nature significantly distinguishes permutation training from other iterative training methods like Adam.

\subsection{Weight configuration and main theorems}
We apply the permutation training on the second hidden layer's weights $\theta^{(2n)}$, leading to the following configuration: the coefficient vector $\theta^{(2n)}$ is permuted from a predetermined vector $W^{(2n)} \in \mathbb{R}^{2n}$, \emph{i.e.}, $\theta^{(2n)} = \tau (W^{(2n)})$.
Without loss of generality, we consider the target continuous function $f^* \in C([0,1])$. Our results begin with a basic scenario with equidistantly distributed location vector $B^{(n)}$ and pairwise coefficient vector $W^{(2n)}$.
The UAP of a permutation-trained network to $f^*$ can be stated as follows:
\begin{theorem}[UAP with a linear layer]\label{th:main1}
    For any function $f^* \in C([0,1])$ and any small number $\varepsilon>0$, there exists a large integer $n \in \mathbb{Z}^+$, 
    and $\alpha,\gamma \in \mathbb{R}$ for $f^{\text{NN}}$ in Eq.~(\ref{NN}) with equidistantly distributed $B^{(n)} = ( b_i )_{i = 1}^n := \left( 0, \tfrac{1}{n-1}, \cdots, 1 \right)$ and corresponding $W^{(2n)} = ( \pm b_i )_{i = 1}^n := (+b_1,-b_1,...,+b_n,-b_n)$, along with a permuted coefficients $\theta^{(2n)} = \tau(W^{(2n)})$, such that $| f^{\text{NN}}(x) - f^*(x) | \le \varepsilon$ for all $x \in [0,1]$.
\end{theorem}
The intuition of this result comes from the rich expressive possibility of permutation training.
Next, we enhance the result in Theorem \ref{th:main1} to a purely permuted situation, suggesting the UAP can be achieved without changing $\alpha, \gamma$ in Eq.~(\ref{NN}).
\begin{theorem}[UAP without the linear layer]\label{th:main2}
    Let $\alpha=0,\gamma=1$.
    For any function $f^* \in C([0,1])$ and any small number $\varepsilon>0$, there exists a large integer $n \in \mathbb{Z}^+$,
    for $f^{\text{NN}}$ in Eq.~(\ref{NN}) with equidistantly distributed $B^{(n)} = ( b_i )_{i = 1}^n:=\left( 0, \tfrac{1}{n-1}, \cdots, 1 \right)$ and $W^{(2n)} = ( \pm b_i )_{i = 1}^n := (+b_1,-b_1,...,+b_n,-b_n)$, along with a permuted coefficients $\theta^{(2n)} = \tau(W^{(2n)})$ such that $| f^{\text{NN}}(x) - f^*(x) | \le \varepsilon$ for all $x \in [0,1]$.
\end{theorem}
Although Theorem \ref{th:main2} considers a theoretically stronger setting, the additional requirement of $n$ reveals the practical meanings of learnable $\alpha, \gamma$ in reducing the necessary network width to achieve UAP.
Moreover, the result can be generalized to the scenario with pairwise random initialization, which is stated by the following theorem.
\begin{theorem}[UAP for randomly initialized parameters] \label{th:random}
    Given a probability threshold $\delta \in (0,1)$, for any function $f^* \in C([0,1])$ and any small number $\varepsilon>0$, there exists a large integer $n \in \mathbb{Z}^+$, and $\alpha,\gamma \in \mathbb{R}$ for $f^{\text{NN}}$ in Eq.~(\ref{NN}) with randomly initialized $B_{rand}^{(n)} \sim \mathcal{U} [0,1]^n$ and pairwisely randomly initialized $W_{rand}^{(2n)} = ( \! \pm p_i )_{i = 1}^n$, $p_i \sim \mathcal{U} [0,1]$, along with a permuted coefficients $\theta^{(2n)} = \tau(W_{rand}^{(2n)})$, such that with probability $1 - \delta$, $| f^{\text{NN}}(x) - f^*(x) | \le \varepsilon$ for all $x \in [0,1]$.
\end{theorem}

\subsection{Proof ideas}
To identify the UAP of our network Eq.~(\ref{NN}) in $C([0,1])$, we employ a piecewise constant function, which is a widely-used continuous function approximator \citep{stein2009real}, and can be expressed as a summation of several step functions. Next, we demonstrate that our networks can approximate each step function.
In this spirit, our constructive proof includes the following three steps (illustrated in Fig.~\ref{fig:main}):
\begin{itemize}
    \item[1.] Approach the target function $f^*$ by a piecewise constant function $g$;
    \item[2.] Approximate each step functions of $g$ by a subnetwork within $f^{\text{NN}}$;
    \item[3.] Annihilate the impact of the remaining parts of $f^{\text{NN}}$.
\end{itemize}

\begin{figure*}[ht]
    \centering
    \includegraphics[height = 4.9cm]{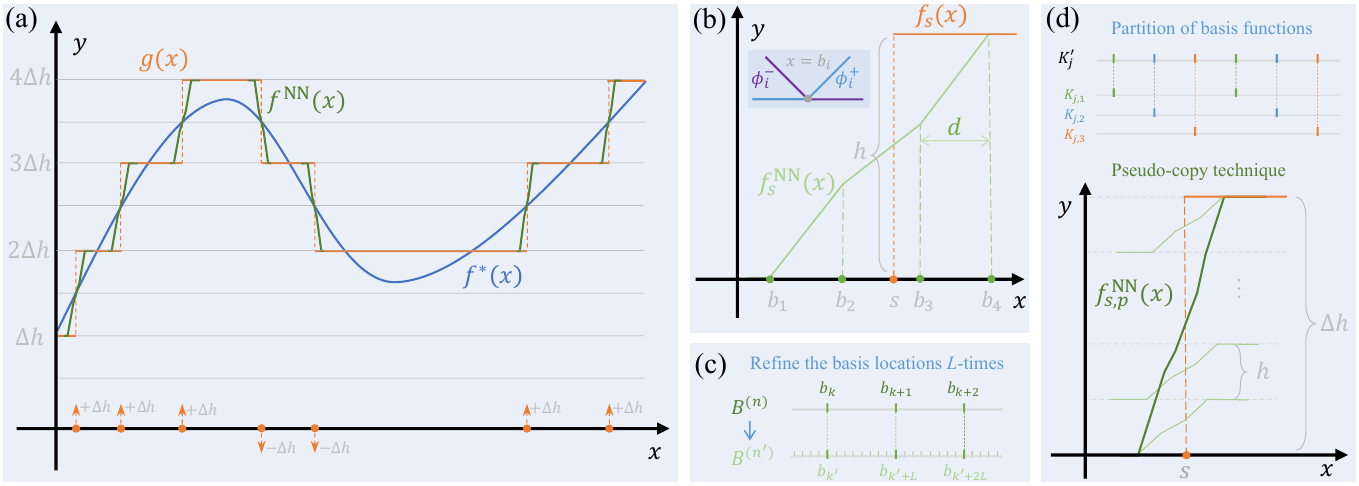}
    \caption{Main idea of the construction. (a) Approximate the continuous function $f^*$ by a piecewise constant function $g$ which is further approximated by permuted networks $f^{\text{NN}}$. (b) The step function approximator $f_s^{\text{NN}}$ constructed by step-matching. (c) Refine the basis functions $L$-times. (d) Stacking pseudo-copies to achieve the desired height.}
    \label{fig:main}
\end{figure*}

Step 1 can be achieved by cutting the range of $f^*$ into subregions with an equal width $\Delta h$, and then constructing a step function at each point where $f^*$ crosses a boundary of these subregions, leading to an approximator with an error $\Delta h$ (illustrated in Fig.~\ref{fig:main}(a)). The statement is outlined below.

\begin{lemma} \label{th:lemma-g}
    For any function $f^* \in C([0,1])$ and any small number $\varepsilon' > 0$, there is a piecewise constant function $g$ with a common jump $\Delta h \le \varepsilon' $, such that $| g(x) - f^*(x) | \le \varepsilon' $ for all $x \in [0,1]$. Moreover, the function $g$ can be written as a summation of $J$ step functions $\{ f_{s_j} \}_{j = 1}^J$ as the following form,
    \begin{equation} \label{eq:g}
        g(x) = \sum_{j = 1}^J a_j f_{s_j}(x)
        =
        \sum_{j = 1}^J a_j \Delta h \chi(x-s_j), \quad a_j = \pm 1, \: s_j \in [0,1], \: J \in \mathbb{Z}^+,
    \end{equation}
    Here $J$ is the step number, $s_j$ is the step location, $a_j$ is the step sign controlling the direction, and $\chi$ is the standard step function
\begin{equation*}
    \chi(x) =
    \left\{ \begin{array}{ll}
        0, & x < 0, \\
        1, & x \ge 0.
        \end{array} \right.
\end{equation*}
\end{lemma}

\begin{proof}
    The function $g$ can be constructed explicitly. According to the Stone-Weierstrass theorem \citep{stone1948generalized}, we assume $f^*$ to be a polynomial function for simplicity. Let the range of $f^*$ be covered by an interval $[k_{\text{min}} \Delta h, k_{\text{max}} \Delta h]$ with two integers $k_{\text{min}}, k_{\text{max}} \in \mathbb{Z}$. 
    Then denote $J$ as the number of intersections between $f^*$ and parallel lines $y = (k + 0.5) \Delta h$, $k = k_{\text{min}}, k_{\text{min}} + 1, \cdots, k_{\text{max}} - 1$. Since $f^*$ is a polynomial function, $J$ can be assumed to be finite (Otherwise $f^*$ will be a constant function due to the fundamental theorem of algebra \citep{hungerford2012algebra}, making it easily approximated by $g$). The scenario with $J=0$ is also trivial since it indicates that $f^*$ lies in $[(k_0-0.5) \Delta h, (k_0+0.5) \Delta h]$ for some integer $k_0 \in [k_{\text{min}}, k_{\text{max}}]$, such that $f^*$ can be approached by the constant function $y = k_0 \Delta h$ with an error $\Delta h$.

    Hence, for any $j = 1, \cdots, J$, we choose $s_j \in [0,1]$ such that $f^*(s_j) = (k_j+0.5) \Delta h$ for some $k_j \in \mathbb{Z}$. 
    The step locations $s_1 < \cdots < s_J$ are distinct since any repetition would contradict the continuity of the target function $f^*$.
    The step sign $a_j$ is determined according to the values of $f^*$ on $[s_{j-1},s_{j+1}]$. It is easy to verify such construction satisfies our requirements.
\end{proof}

The execution of step 2 is inspired by the divide-and-conquer algorithm in computer science \citep{hopcroft1983data}.
For each step function $f_{s_j}$ in $g$, we select basis functions to construct a step function approximator $f_{s_j}^{\text{NN}}$, then sum them up to approach $g$.
This \emph{step-matching} construction utilizes four pairs of basis functions $\{\pm b_i\}_{i=1}^4$ (shown in Fig.~\ref{fig:main}(b)), and establishing a one-to-one mapping between coefficients and biases, \emph{i.e.}, $\{p_i, q_i\}_{i=1}^4=\{\pm b_i\}_{i=1}^4$. It ensures that each coefficient and location are uniquely assigned and prevents conflict between different approximators.

Step 3 plays a vital role in the proof construction, and serves as an essential distinguishing factor that sets permutation training apart from conventional scenarios.
Note that the specific setting of permutation training poses a crucial challenge that the proof must utilize every parameter, rather than just pick up the desired parameters and discard the rest.
Therefore, it is essential to eliminate the remaining network parameters after step 2 to prevent the potential accumulation of errors.
We solve this problem by proposing a \emph{linear reorganization} to write the remaining part as a linear function with a controllable slope.

To further enhance the conclusion of Theorem \ref{th:main1} to Theorem \ref{th:main2}, we first introduce a technique called \emph{pseudo-copy}, which can achieve UAP without altering $\gamma$ in Eq.~(\ref{NN}). By partitioning the refined basis functions, several pseudo-copies $f_{s, p}^{\text{NN}}$ of the original approximator $f_{s}^{\text{NN}}$ can be constructed with a controllable error. The final height can then be achieved by stacking these copies together rather than changing $\gamma$ (see Fig.~\ref{fig:main}(d)). Additionally, to make the shift factor $\alpha$ removable, our \emph{constant-matching} construction can provide the necessary shifting. It also enables another way to eliminate the remaining part of the network.

Extending the UAP to the random initializations is justified by the fact that
the parameters randomly sampled from uniform distributions become denser, thus approaching the equidistant case.
Therefore, a sufficiently wide network has a high probability of finding a subnetwork that is close enough to the network with UAP in the equidistant case.
Then this subnetwork can also achieve UAP due to its continuity.
The remaining part of the network can be eliminated by step 3.

\section{UAP of permutation-trained networks}
\label{sec:proof}

This section provides a detailed construction of the approximator with weight-permuted networks in the equidistant case, along with an estimation of the convergent rate of approximation error. The extension to the scenario with random initialization is also thoroughly discussed.

\subsection{The construction of step, constant, and linear function approximators}
Lemma~\ref{th:lemma-g} enables us to approximate the target function $f^*$ by a piecewise constant function $g$ in Eq.~(\ref{eq:g}). Next, we aim to approximate each step function $f_{s_j}$ within $g$, respectively, and then eliminate the remaining part of the network. This section introduces several key constructions to achieve this goal in the equidistant case.

\subsubsection{Step-matching construction of step function approximators $f_s^{\text{NN}}$}
Here we construct the step function approximator $f_{s}^{\text{NN}}$ for a given step function $f_s(x) = \Delta h \, \chi(x-s)$ with height $\Delta h$ and location $s$.
The construction considers four pairs of basis functions $\{\phi_i^\pm\}_{i=1}^4$ with locations $\{b_i\}_{i=1}^4$
and coefficients $\{p_i, q_i\}_{i=1}^4=\{\pm b_i\}_{i=1}^4$, which has the following form,
\begin{equation} \label{four_pairs}
    f_{s}^{\text{NN}}(x) = \sum_{i = 1}^4 p_i \phi_i^+ (x) + \sum_{i = 1}^4 q_i \phi_i^- (x), \quad x \in [0,1].
\end{equation}
Here we require the locations $b_1 < b_2 < b_3 < b_4$ satisfy the following symmetric condition,
\begin{equation} \label{eq:d}
    d:=b_2-b_1=b_4-b_3 \, \Longrightarrow \, b_1 + b_4 = b_2 + b_3,
\end{equation}
where $d$ is the basis distance.
To ensure a local error of the approximator, we appeal $f_{s}^{\text{NN}}$ to be $x$-independent outside the interval $[b_1,b_4]$. As a result, the coefficients $p_i,q_i$ must satisfy $\sum_{i = 1}^4 p_i = \sum_{i = 1}^4 q_i = 0$, which implies the correspondence between $\{ p_i, q_i \}_{i = 1}^4$ and $\{ \pm b_i \}_{i = 1}^4$ as 
\begin{equation}\label{eq:coefficients_step}
    \begin{split}
        p_1 = -b_1, \quad p_2 = +b_2, \quad p_3 = +b_3, \quad p_4 = -b_4, \\
        q_1 = +b_4, \quad q_2 = -b_3, \quad q_3 = -b_2, \quad q_4 = +b_1, 
    \end{split}
\end{equation}
We call the $\{ p_i, q_i \}_{i = 1}^4$ and $\{ \pm b_i \}_{i = 1}^4$ is \emph{step-matching} if they satisfy Eq.~(\ref{eq:coefficients_step}), which gives the piecewise form of $f_s^{\text{NN}}$ in Eq.~(\ref{four_pairs}) as
\begin{equation} \label{f_n_piece}
    \begin{split}
        f_s^{\text{NN}}(x) = 
        \left\{ \begin{array}{lr}
            2 b_1 b_4 - 2 b_2 b_3 & 0 \le x < b_1, \vspace{0.4ex} \\
            (- b_1 + b_4)x + b_1^2 + b_1 b_4 - 2 b_2 b_3 & b_1 \le x < b_2, \vspace{0.4ex} \\
            (-2 b_1 + 2 b_2)x + b_1^2 - b_2^2 + b_1 b_4 - b_2 b_3 & b_2 \le x < b_3, \vspace{0.4ex} \\
            (- b_1 + b_4)x + b_1^2 - b_2^2 -b_3^2 + b_1 b_4 & b_3 \le x < b_4, \vspace{0.4ex} \\
            b_1^2 - b_2^2 - b_3^2 + b_4^2 & b_4 \le x \le 1. \text{ }
            \end{array} \right. 
    \end{split}
\end{equation}
The profile of this $f_s^{\text{NN}}$ can be found in Fig.~\ref{fig:main}(b), which shows that $f_s^{\text{NN}}$ is monotone 
and can approach a step function with the height $h$ satisfying the following relation,
\begin{equation} \label{eq:height}
    h = 2(b_1^2 - b_2^2 - b_3^2 + b_4^2) = 4d(b_4-b_2).
\end{equation}
Notice that choosing the basis functions adjacently will lead to $d = \frac{1}{n-1}$ and $h = 8d^2$.
By shifting $h/2$ and scaling $\Delta h/h$, we use $f_s^{\text{NN}}$ to approach step function $f_s$ with $s \in [b_1,b_4]$.
It is obvious that the $L^\infty$ error has the following trivial bound,
\begin{equation} \label{eq:fs_error}
    \left| \frac{\Delta h}{h} \left[ f_s^{\text{NN}}(x) + \frac{h}{2} \right] - f_s(x) \right| \le \Delta h, \quad \forall \, x \in [0,1].
\end{equation} 
Eq.~(\ref{f_n_piece}) implies that the approximation is exactly accurate when $x \notin [b_1, b_4]$. A toy example of this step-matching construction is shown in \ref{app:toy}.

\begin{remark}
     Although \citet{cai2021least} suggested that a step function could be approached by two ReLU basis functions, this approach is unsuitable for the permutation training scenario.
     For an error tolerance $\varepsilon$, a step function $f_s(x) = h_s \chi(x-s)$ can be well approximated by a linear combination of two ReLU basis functions as
     \begin{equation*}
         \tilde{f}_s^{\text{NN}}(x) = \frac{h_s}{2\varepsilon} [\text{ReLU}(x - s + \varepsilon) - \text{ReLU}(x - s - \varepsilon)].
     \end{equation*}
     However, the dependence of the coefficients on the step height $h_s$ hinders further construction since each coefficient can be used only once.
\end{remark}

\subsubsection{Constant-matching construction of constant function approximators $f_c^{\text{NN}}$}

The four-pair form in Eq.~(\ref{four_pairs}) can also be utilized to approximate a constant function, which plays a crucial role in the proof of Theorem~\ref{th:main2}. The constant function approximator (denoted as $f_c^{\text{NN}}$) shares the same form of $f_s^{\text{NN}}$ in Eq.~(\ref{four_pairs}) but differs slightly in the parameter assignment in Eq.~(\ref{eq:coefficients_step}). 
The coefficients $\{p_i, q_i\}_{i = 1}^4$ are set to equalize the height of the two constant pieces $x < b_1$ and $b_4 \le x$, leading to $-\sum_{i = 1}^4 p_i b_i = \sum_{i = 1}^4 q_i b_i$. A possible choice is
\begin{equation} \label{eq:coefficients_constant}
    \begin{split}
        p_1 = -b_1, \quad p_2 = +b_2,  \quad p_3 = +b_3, \quad p_4 = -b_4, \\
        q_1 = +b_1, \quad q_2 = -b_2, \quad q_3 = -b_3, \quad q_4 = +b_4.
    \end{split}
\end{equation}
We call the $\{ p_i, q_i \}_{i = 1}^4$ and $\{ \pm b_i \}_{i = 1}^4$ is \emph{constant-matching} if they satisfy Eq.~(\ref{eq:coefficients_constant}).
With these chosen coefficients $\{ p_i, q_i \}_{i = 1}^4$ and the the symmetry of coefficients in Eq.~(\ref{eq:d}), the constant function approximator $f_c^{\text{NN}}$ can be written as:
\begin{equation} \label{eq:f_c}
    \begin{split}
        f_c^{\text{NN}}(x) & = 
            (b_2 + b_3 - b_1 - b_4) \, x + b_1^2 - b_2^2 - b_3^2 + b_4^2 \\
            & = 2d(b_4 - b_2) = h/2, \quad x \in [0, 1].
    \end{split}
\end{equation}
By the relations of $h$ in Eq.~(\ref{eq:height}), it gives a representation of constant $C = h/2$ without the approximation error, i.e., $| f_{c}^{\text{NN}}(x) - h/2 | = 0$ for all $x \in [0,1]$.

Furthermore, by changing the sign of $\{ p_i, q_i \}_{i = 1}^4$ in Eq.~(\ref{eq:coefficients_constant}) simultaneously, an approximator $f_{-c}^{\text{NN}}$ with a negative constant $-C = -h/2$ can also be constructed:
\begin{equation*} 
        f_{-c}^{\text{NN}}(x) = - 2d(b_4 - b_2) = -h/2, \quad x \in [0, 1].
\end{equation*}
Therefore, we can freely adjust the sign $a_j = \pm 1$ of the approached constant. It also enables us to construct $f_{c}^{\text{NN}}$ and $f_{-c}^{\text{NN}}$ with the same constant separately, then pair them up to offset with each other, \emph{i.e.}, $f_{c}^{\text{NN}}(x) + f_{-c}^{\text{NN}}(x) = 0$ for all $x \in [0,1]$. A toy example of this constant-matching construction is shown in \ref{app:toy}.

\subsubsection{Linear reorganization of the linear function approximators $f_\ell^{\text{NN}}$}
Apart from the four-pair construction discussed before, we can also construct a linear function $f_\ell^{\text{NN}}$ by reorganizing each pair of basis functions. 
Concretely, we choose a pair of basis functions located at $b_i$ and set the coefficients as $\{p_i, q_i\}=\{\pm b_i\}$. The linear function approximator $f_\ell^{\text{NN}}$ can be written as the following form,
\begin{equation} \label{eq:unused_linear}
    f_\ell^{\text{NN}}(x) = m_i \ell_i, \quad \ell_i(x) := b_i \phi_i^+(x) - b_i \phi_i^-(x) = b_i x - b_i^2, \quad x \in [0,1],
\end{equation} 
where $m_i = \pm 1$ is a freely adjusted sign.
We call this method as \emph{linear reorganization}. It allows us to adjust the number of the total basis functions, as the $L^\infty$-norm of $f_\ell^{\text{NN}}$ located at $b_i$ has an upper bound of $b_i^2$. Hence we can choose the unwanted $b_i$ that is small enough, then omit it along with the related coefficients $\{p_i, q_i\}=\{\pm b_i\}$ from the network without affecting the approximation error.

\subsection{Annihilate the unused part of the network} \label{sec:unused}

After constructing approximators to approach the target function, the permutation training setting requires that the remaining parameters be suitably arranged to eliminate their impact.
Notice that a pair of basis functions $\phi_i^{\pm}$ 
are either used together or not at all. We apply the linear reorganization in Eq.~(\ref{eq:unused_linear}) to form them into linear functions $f_{\ell_i}^{\text{NN}}$. 
Concretely, denote $I_{\text{un}} \subset \{1,2,...,n\}$ as the index of unused basis functions and $\bar n$ as the number of elements in $I_{\text{un}}$, which can be processed to be an even number.
After taking a sum of all $i \in I_{\text{un}}$, the resulting $\mathcal{S_\ell}$ has the following linear form,
\begin{equation} \label{eq:minS_l2}
    \mathcal{S_\ell}(x) = \sum_{i \in I_{\text{un}}} f_{\ell_i}^{\text{NN}}(x) = \sum_{i \in I_{\text{un}}} m_i b_i x - \sum_{i \in I_{\text{un}}} m_i b_i^2 =: \beta x + \eta, \quad x \in [0,1],
\end{equation}
where $\beta$ is the slope and $\eta$ is the intercept. The goal then is to control the $L^\infty$-norm of $\mathcal{S_\ell}$. We first choose a proper $m_i \in \{-1, 1\}$ for each $i \in I_{\text{un}}$ to reduce $|\beta|$, which is equivalent to assigning addition and subtraction operations within a given sequence to reduce the final result's absolute value. Motivated by the \emph{Leibniz's test} (known as \emph{alternating series test}) \citep{rudin1953principles}, the following lemma offers a solution with an upper bound of $\beta$.

\begin{lemma} \label{th:Leibniz}
    For an even number $\bar n$ and a sequence of real number $c_i \in \mathbb{R}, i = 1, \cdots, \bar n$,
    there exists a choice of $m_i \in \{ -1, 1\}$, $i = 1, \cdots, \bar n$, such that 
    \begin{align}
        0 \le \sum_{i = 1}^{\bar n} m_i c_i \le \Delta c,
        \quad
        \Delta c = \max_{\substack{1 \le j \le \bar n \\ j \notin \operatorname{argmax} c_j}} \left( \min_{\substack{i \ne j \\ c_i \ge c_j}} (c_i - c_j) \right),
    \end{align}
    where $\Delta c$ can be regarded as the largest gap between the adjacent elements after sorting $c_i$ in descending order (see \ref{app:Leibniz} for a toy example).
\end{lemma}
\begin{proof}
    Without loss of generality, we assume $c_1 \ge c_2 \ge \cdots \ge c_{\bar n} \ge 0$ and thus $\Delta c = \max_{i} | c_{i} - c_{i-1}|$.
    Since $\bar n$ is an even number, we define a new series $\{r_j\}_{j = 1}^{\bar n / 2}$ as the difference between each pair of elements in $\{c_i\}_{i = 1}^{\bar n}$,
    \begin{equation*}
        r_j = c_{2j-1} - c_{2j} \ge 0, \quad j = 1, 2, \cdots, \bar n / 2.
    \end{equation*}
    Then we permute $( r_j )_{j = 1}^{\bar n / 2}$ to be descending order. Concretely, we find a permutation $\tau$ and denote $r'_j = \tau(r_j) = [\tau(c_{2j-1}) - \tau(c_{2j})]$, such that $r'_1 \ge r'_2 \ge \cdots \ge r'_{\bar n / 2} \ge 0$.
    Next, we alternatively use addition and subtraction operations on each $r_j'$ by noting $\lambda_j=(-1)^{j-1}$, $j = 1, \cdots, \bar n / 2$, and estimate the summation $\mathcal{S}_{\bar n / 2}$ as the following forms,
    \begin{equation*}
        \mathcal{S}_{\bar n / 2} :=
        \sum_{j = 1}^{\bar n / 2} \lambda_j r'_j
        =
        \left\{ \begin{array}{l}
                (r_1' - r_2') + (r_3' - r_4') +  \cdots  \ge 0,\\
                r_1' - (r_2' - r_3') - (r_4'-r_5') - \cdots  \le r_1' \le \Delta c.
        \end{array} \right. 
    \end{equation*}
    Note that $\mathcal{S}_{\bar n / 2} = \sum_{j = 1}^{\bar n / 2} \lambda_j [\tau(c_{2j-1}) - \tau(c_{2j})]$ is of the form of $\sum_{i = 1}^{\bar n} m_i c_i$, hence the choice of $m_i$ implied by $\lambda_j$ satisfies our requirement and the proof is finished 
\end{proof}

\begin{remark}
    Leibniz's test guarantees that an alternating series that decreases in absolute value is bounded by the leading term. However, we emphasize that utilizing Leibniz's test directly to $\{c_i\}_{i = 1}^{\bar n}$ can only get a trivial bound. Therefore, the introduction of $r_j$ is important to get the desired result.
\end{remark}

After applying Lemma~\ref{th:Leibniz}
to control the slope $\beta$, the intercept $\eta$ is consequently determined by the chosen $\{m_i\}_{i \in I_{\text{un}}}$. Thus, we can choose a constant to adjust the intercept, which establishes an upper bound for the error $S_{\ell}$ of the remaining part.

\subsection{Proof~of~Theorem \ref{th:main1}}

Lemma \ref{th:lemma-g} offers a piecewise constant function $g$ to approach the target function $f^*$. Now we prove that by permuting the selected coefficients and eliminating the remaining part, $f^{\text{NN}}$ can approximate piecewise constant function $g$ with the same accuracy, enabling us to prove the Theorem \ref{th:main1}. 

\begin{proof}[Proof~of~Theorem \ref{th:main1}]
    For any target function $f^* \in C([0,1])$ and small number $\varepsilon$, we aim to construct a network $f^{\text{NN}}$ with $n$ pairs of basis functions to approximate $f^*$ with error $\varepsilon$.
    The goal is achieved with the following points: 

\textbf{a)}  Approach $f^*$ by a piecewise constant function $g$.
Employing Lemma \ref{th:lemma-g} by letting $\varepsilon' = \varepsilon/4$, we can construct a piecewise constant function $g$ in Eq.~(\ref{eq:g}) with a constant height $\Delta h \le \varepsilon' = \varepsilon/4$, distinct step locations $s_1 < \cdots < s_J$.
It gives $|g(x)-f^*(x)|\le \Delta h < \varepsilon/2$ for all $x\in [0,1]$.
Here we denote $\delta_s := \max_j |s_j-s_{j-1}|$ as the maximal gap of step locations, where $j = 0, 1, \cdots, J+1$ and $s_0 = 0$, $s_{J+1} = 1$.

\textbf{b)}  Approximate each step function in $g$ by the step-matching approximator $f_{s_j}^{\text{NN}}$ in Eq.~(\ref{four_pairs}).
Since the step locations $\{ s_j \}_{j = 1}^J$ are distinct,
we can choose a network $f^{\text{NN}}$ with large enough $\hat n$, \emph{i.e.}, the locations in $B^{(\hat n)}$ are dense enough, such that there are enough basis functions to construct $f_{s_j}^{\text{NN}}$ for each $a_j f_{s_j}$, respectively. 
In fact, for any $\hat n > {8}/{\delta_s} + 1$, the distance between basis functions $\hat d = 1/(\hat n-1)$ satisfies $\hat d < \delta_s / 8$. 

However, to maintain the consistency of the following estimation, we refine the basis locations $B^{(\hat n)}$ to $B^{(n)}$ with $n = L (\hat n - 1) + 1$ for some integer $L \ge \Delta h (\hat n - 1)^2 / 8$ (see Fig.~\ref{fig:main}(c)). 
Denote $K = \{1, \cdots, n\}$ as the index set of the locations $B^{(n)}$, and for each $j = 1, \cdots, J$, we choose the basis functions with the index $K_{j} :=\{k_j, k_j + L, k_j + 2L, k_j+ 3L \} \subset K$, such that $s_j \in [b_{k_j+L}, b_{k_j+2L})$ and $\{ K_{j} \}_{j = 1}^J$ has empty intersection.
Therefore, for each $a_j f_{s_j}$, an approximator $ (f_{s_j}^{\text{NN}} + a_j h/2)$ with a shifting $a_j h/2$ can be constructed by applying the step-matching construction on $\{ p_{k}, q_{k} \}_{k \in K_j}$. This approximation has the following error estimation given by Eq.~(\ref{eq:fs_error}):
\begin{equation} \label{th1-1}
    \left|  \left[ f_{s_j}^{\text{NN}}(x) + a_j \frac{h}{2} \right] - a_j \frac{h}{\Delta h} f_{s_j}(x) \right|
    \le h, \quad \forall \, x \in [0,1], \quad j = 1, \cdots, J,
\end{equation}
where $h=8 \hat d^2 = 8/(\hat n-1)^2 = 8L^2 / (n - 1)^2$ is the height determined by $\hat n$, and the scaling $h/\Delta h$ serves to match the constant height $\Delta h$ and its approximator height $h$. Our requirement of $L$ gives the condition $L \ge \Delta h/ h$. 
Furthermore, denote $K_{\text{use}} = \cup_{j=1}^J K_{j}$ as the index set for all involved basis functions in Eq.~(\ref{th1-1}), which leads to the requirement of $\hat n > 4J$. We define $\mathcal{S}_{\text{use}}$ as the picked subnetwork of $f^{\text{NN}}$,
\begin{equation*}
    \mathcal{S}_{\text{use}}(x) := 
    \sum_{k \in K_{\text{use}}} \left[ p_k \phi_k^+ (x) + q_k \phi_k^- (x) \right]
    =
    \sum_{j=1}^J f_{s_j}^{\text{NN}}(x), \quad x \in [0,1],
\end{equation*}
then the properties of our step-matching construction and Eq.~(\ref{th1-1}) imply the error $E_{\text{use}}$ of the summed approximators can be estimated as the following form,
\begin{equation} \label{eq:use_error}
    \begin{split}
        E_{\text{use}}
    := &\, 
    \max_{x \in [0,1]} \left| \mathcal{S}_{\text{use}}(x) + \frac{h}{2} \sum_{j = 1}^{J} a_j - \frac{h}{\Delta h} g(x) \right|
    \\= &\,
    \max_{x \in [0,1]} \left| \sum_{j=1}^J \left[ f_{s_j, p_l}^{\text{NN}}(x) + a_j \frac{h}{2} \right] - \sum_{j=1}^{J} \frac{h}{\Delta h} a_j f_{s_j}(x) \right|
    \le h.
    \end{split}
\end{equation}
Here $\sum_{j = 1}^{J} a_j =: J' \le J$ is a constant, and height $h$ satisfies $h = 8L^2 / (n - 1)^2$.

\textbf{c)}  Annihilate the impact of the unused part.
The linear reorganization in Eq.~(\ref{eq:unused_linear}) enables write the unused part in $f^{\text{NN}}$ into a linear function $\mathcal{S_{\text{un}}}$, namely,
\begin{equation*}
    \mathcal{S_{\text{un}}}(x) = \sum_{k \in K \setminus K_{\text{use}}} \left[ p_k \phi_k^+ (x) + q_k \phi_k^- (x) \right] = \sum_{k \in K \setminus K_{\text{use}}} (m_k b_k x - m_k b_k^2) =: \beta x + \eta,
\end{equation*}
Since the number of unused basis functions 
can be processed to be even, we apply Lemma \ref{th:Leibniz} on the series $ \{ b_{k} \}_{k \in K \setminus K_{\text{use}}}$ to get a choice of $m_{k} \in \{-1, 1\}$ for each $k \in K \setminus K_{\text{use}}$, which provides an upper bound of the slope $\beta$ as $0 \le \beta \le \Delta b$,
where $\Delta b$ is the largest gap between the adjacent locations in $\{b_{k}\}_{k \in K \setminus K_{\text{use}}}$. 
Notice that in $f^{\text{NN}}$, the refined basis distance $d < 1/L \hat n$ is small enough to ensure that there is at least one unused basis function between two adjacent used basis functions, \emph{i.e.}, $\Delta b \le 2/L \hat n$.
To control the error of the intercept $\eta$, we introduce a shifting $C_\eta = -\eta$, thus the error $E_{\text{un}}$ introduced by the unused part gives the following estimation, 
\begin{equation} \label{eq:minEqui}
    E_{\text{un}} := \max_{x \in [0,1]} \big| \mathcal{S_{\text{un}}}(x) + C_\eta \big| \le \Delta b < \frac{2}{L \hat n} \le \frac{2 h}{\Delta h \hat n}.
\end{equation}
Therefore, we choose $\hat n > 2 / \Delta h$ such that the error $E_{\text{un}}$ satisfies $E_{\text{un}} \le h$.

\textbf{d)} Complete the proof by choosing the suitable values of $\gamma$, $\alpha$.
We choose $\hat n$ and two factors $\gamma$, $\alpha$, such that
\begin{equation*}
    \hat n \ge \max \left\{ 4J, \: \frac{8}{\delta_s} + 1, \: \frac{2}{\Delta h} \right\}, \quad \gamma = \frac{\Delta h}{h}, \quad \alpha = \Delta h \left( \frac{J'}{2} + \frac{C_\eta}{h} \right).
\end{equation*}
Moreover, Let $n = L (\hat n - 1) + 1$ for some integer $L \ge \Delta h (\hat n - 1)^2 / 8$, 
then Eq.~(\ref{eq:minEqui})-(\ref{eq:pc_error}) implies that the approximation error of the network $f^{\text{NN}}$ with width $n$ satisfies
\begin{equation*} 
    \begin{split}
    \left| f^{\text{NN}}(x) - g(x) \right| 
    & = \left| \gamma \sum_{k \in K} \left[ p_k \phi_k^+ (x) + q_k \phi_k^- (x) \right] + \alpha - g(x) \right| \\
    & \le \left| \frac{\Delta h}{h} \big[ \mathcal{S}_{\text{use}}(x) + \mathcal{S}_{\text{un}}(x) \big] + \left( \frac{\Delta h}{2} J' + \frac{\Delta h}{h} C_\eta \right) - g(x) \right| \\
    & \le \frac{\Delta h}{h} \left| \mathcal{S}_{\text{use}}(x) + \frac{h}{2} J' - \frac{h}{\Delta h} g(x) \right| + \frac{\Delta h}{h} \big| \mathcal{S}_{\text{un}}(x) + C_\eta \big| \\
    & \le \frac{\Delta h}{h} (E_{\text{use}} + E_{\text{un}}) \le 2 \Delta h \le \varepsilon / 2, \quad \forall \, x \in [0, 1].
    \end{split}
\end{equation*}
This complete the proof since $|f^{\text{NN}}(x)-f^*(x)|\le \varepsilon$ for all $x\in [0,1]$.
\end{proof}

\subsection{Proof~of~Theorem \ref{th:main2}}  \label{sec:without_scaling}

Next, we remove scaling $\gamma$ and shifting $\alpha$ during the proof, \emph{i.e.,} achieve UAP with fixed factors $\gamma=1,\alpha=0$.
To eliminate the scaling $\gamma$,
we introduce the \emph{pseudo-copy} technique to let $\gamma=L$ for an integer $L \in \mathbb{Z}^+$, allowing copy the approximator $L$-times and stack them (instead of multiplying by $\gamma$) to match the desired height. 
The shifting $\alpha$ can be replaced by our constant-matching construction in Eq.~(\ref{eq:f_c}), which is also applied to eliminate the remaining parameters, since our linear reorganization in Eq.~(\ref{eq:unused_linear}) requires a shifting $C_\eta$ to control the error.

\begin{proof}[Proof~of~Theorem \ref{th:main2}]

For a given target function $f^* \in C([0,1])$, we first apply Theorem \ref{th:main1} to obtain a network $f^{\text{NN}}$ to approximate $f^*$, and then enhance the construction by pseudo-copy technique and constant-matching construction to eliminate the scaling $\gamma$ and shifting $\alpha$, leading to a network $f^{\text{NN}}$ in Eq.~(\ref{NN}) with fixed factors $\gamma=1,\alpha=0$. To facilitate distinction, we will denote this network equipped with pseudo-copy technique as $f_{\text{pc}}^{\text{NN}}$ in the following text.

\textbf{a)}  Approximate the target function $f^*$ by a network $f^{\text{NN}}$ with learnable factors $\gamma$ and $\alpha$.
By applying Theorem \ref{th:main1}, we construct $f^{\text{NN}}$ to approximate $f^*$ through a piecewise constant function $g$ with error $\varepsilon$. 
Following the previous discussion, we have $\Delta h \le \varepsilon / 8$, $\gamma = \Delta h / h$. During the construction, $L$ are chosen to be $L = \gamma$.

\textbf{b)}  Remove the scaling $\gamma$ by the pseudo-copy technique.
We first reassign the basis functions in $f^{\text{NN}}$ to construct the pseudo-copies of each approximator $f_{s_j}^{\text{NN}}$.
For each $K_{j} =\{k_j, k_j + L, k_j + 2L, k_j + 3L \} \subset K$, we denote the corresponding adjacent index set for pseudo-copy as $K_j' := \{k_j, k_j+1, \cdots, k_j + 4L-1\}$ such that $K_j \subset K_j'$.
To maintain the same height of each pseudo-copy, we partition $K_j'$ into $L$ subsets as $K_j' = \cup_{l=1}^L K_{j,l}$ by choosing after every $L$ indexes
(illustrated at the top of Fig.~\ref{fig:main}(d)).

For each $l = 1, \cdots, L$, we apply the step-matching construction on $\{ p_k, q_k \}_{k \in K_{j,l}}$ to construct the pseudo-copy $f_{s_j, p_l}^{\text{NN}}$ for $f_{s_j}^{\text{NN}}$ in Eq.~(\ref{th1-1}), respectively. 
Since each $f_{s_j, p_l}^{\text{NN}}$ has the same height $h$ with $f_{s_j}^{\text{NN}}$, we have the copy error as $ | f_{s_j, p_l}^{\text{NN}}(x) - f_{s_j}^{\text{NN}}(x) | < h$ for all $x \in [0,1]$. Therefore, Eq.~(\ref{th1-1}) allows the summed $f_{s_j, p_l}^{\text{NN}}$ to approximate $a_j f_{s_j}$ with the following error:
\begin{equation} \label{eq:pseudo-copy}
    \begin{split}
        & \left| \sum_{l = 1}^L \left[ f_{s_j, p_l}^{\text{NN}}(x) + a_j \frac{h}{2} \right] -  a_j f_{s_j}(x) \right| \\
        \le \, & 
        \sum_{l = 1}^L \left| f_{s_j, p_l}^{\text{NN}}(x) - f_{s_j}^{\text{NN}}(x) \right| + \left| L \left[ f_{s_j}^{\text{NN}}(x) + a_j \frac{h}{2} \right] - a_j f_{s_j}(x) \right| \\
        \le \, & 
        L h + \Delta h = 2 \Delta  h, \quad \forall \, x \in [0, 1], \quad j = 1, \cdots, J.
    \end{split}
\end{equation}
Here $a_j h / 2$ is the shifting required by the pseudo-copies. 

\textbf{c)}  Replace the shifting with constant-matching construction.
After constructing the pseudo-copies $f_{s_j, p_l}^{\text{NN}}$, our constant-matching construction in Eq.~(\ref{eq:f_c}) can pair them up with the necessary shifting $f_{c_j}^{\text{NN}} = a_j h / 2$, enabling the combined $\sum_{l  =1}^L (f_{s_j, p_l}^{\text{NN}} + f_{c_j}^{\text{NN}})$ to approach $a_j f_{s_j}$. 
Denote $K_{\text{use}}'$ as the index set for all involved basis functions. We define $\mathcal{S}_{\text{use}}'$ as the picked subnetwork of $f_{\text{pc}}^{\text{NN}}$,
\begin{equation*}
    \mathcal{S}_{\text{use}}'(x) := 
    \sum_{k \in K_{\text{use}}'} \left[ p_k \phi_k^+ (x) + q_k \phi_k^- (x) \right]
    =
    \sum_{l = 1}^L \sum_{j=1}^J \left[ f_{s_j, p_l}^{\text{NN}}(x) + f_{c_j}^{\text{NN}}(x) \right].
\end{equation*}
where $x \in [0,1]$. Eq.~(\ref{eq:pseudo-copy}) implies the following error estimation,
\begin{equation} \label{eq:pc_error}
    \begin{split}
        E_{\text{use}}'
    := &\, 
    \max_{x \in [0,1]} \left| \mathcal{S}_{\text{use}}'(x) - g(x) \right|
    \\= &\,
    \max_{x \in [0,1]} \left| \sum_{j=1}^J \sum_{l = 1}^L \left[ f_{s_j, p_l}^{\text{NN}}(x) + f_{c_j}^{\text{NN}}(x) \right] - \sum_{j=1}^{J} a_j f_{s_j}(x) \right|
    \le 2 \Delta h.
    \end{split}
\end{equation}
This construction uses $8L$ basis locations, leading to the requirement as $n > 8LJ$. 

\textbf{d)}  Eliminate the unused part of the network. 
Since the linear reorganization in Eq.~(\ref{eq:unused_linear}) is invalid,
we turn to apply the constant-matching construction in Eq.~(\ref{eq:f_c})
 to eliminate the unused part of the network.
Concretely, for the remaining $n - 8JL$ pairs of basis functions in $f_{\text{pc}}^{\text{NN}}$,
we construct $f_{\pm c_t}^{\text{NN}}$, $t = 1, \cdots, \lfloor n/8 \rfloor - JL$, then pairing them up to offset with each other, \emph{i.e.}, 
\begin{equation*}
    \mathcal{S_{\text{un}}'}(x) = \sum_{k \in K \setminus K_{\text{use}}'} \left[ p_k \phi_k^+ (x) + q_k \phi_k^- (x) \right] = \sum_{t = 1}^{\lfloor n/8 \rfloor - JL} \left[ f_{c_t}^{\text{NN}}(x) + f_{-c_t}^{\text{NN}}(x) \right] + R(x),
\end{equation*}
Here $R(x)$ is the residual part since $n$ is not always divisible by 8. However, our linear reorganization in Eq.~(\ref{eq:unused_linear}) enables omit the basis functions with sufficiently small $b_{i}$.
Concretely, the $L^\infty$ error introduced by $R(x)$ can be controlled as $|R(x)| \le \sum_{k = 1}^4 \bar b_{k}^2 =: C_R$ for all $x \in [0,1]$, where $\{\bar b_{k}\}_{k = 1}^4$ is the missed or external basis locations. Since the constant-matching construction has flexibility,
$\{\bar b_{k}\}_{k = 1}^4$ can be small enough to ensure $C_R \le \Delta h$.
Hence, the error $E_{\text{un}}'$ introduced by the unused parameters satisfies 
\begin{equation} \label{eq:E_un'}
    E_{\text{un}}' = \max_{x \in [0,1]} \left| 
        \mathcal{S_{\text{un}}'}(x) \right| = C_R + C_c
         \le 2 \Delta h.
\end{equation}
Here the constant $C_c$ comes from the possible mismatch in pairing up $f_{c_t}^{\text{NN}} + f_{-c_t}^{\text{NN}}$.
However, for a sufficiently wide network, $C_c$ can be small enough as $C_c < \Delta h$.

\textbf{e)}  Complete the proof by combining the previous estimation.
By setting $\gamma=1,\alpha=0$ and choosing a sufficiently large $L \in \mathbb{Z}^+$, such that 
\begin{equation*}
    n = \sqrt{\frac{8L^3}{\Delta h}} + 1 > \max \left\{ 8JL, \: \frac{8L}{\delta_s} + L, \: \frac{2L}{\Delta h} \right\}.
\end{equation*}
Eq.~(\ref{eq:pc_error})-(\ref{eq:E_un'}) gives estimation of the overall approximation error as follows:
 \begin{equation*} 
    \begin{split}
    \left| f_{\text{pc}}^{\text{NN}}(x) - g(x) \right| 
    & = \left| \sum_{k \in K} \left[ p_k \phi_k^+ (x) + q_k \phi_k^- (x) \right] - g(x) \right| \\
    & \le \left| \mathcal{S}_{\text{use}}'(x) - g(x) \right| +  \left| \mathcal{S}_{\text{un}}'(x) \right| \\
    & \le E_{\text{use}}' + E_{\text{un}}' \le 4 \Delta h, \quad \forall \, x \in [0, 1].
    \end{split}
\end{equation*}
where $\Delta h$ is chosen to satisfy $\Delta h \le \varepsilon / 8$. Hence we can finish the proof of Theorem \ref{th:main2} since $|f_{\text{pc}}^{\text{NN}}(x)-f^*(x)|\le \varepsilon$ for all $x\in [0,1]$.
\end{proof}

\subsection{Estimate the approximation rate} \label{sec:error_rate}

Here we estimate the approximation rate of the $L^2$-error $e_{s}$ of approximating single step function $f_s$ in Eq.~(\ref{eq:g}) by the approximator $f_s^{\text{NN}}$ in Eq.~(\ref{f_n_piece}). 
In our four-pair construction, we assume $s=(b_2+b_3)/2$ and introduce $k_1$ and $k_2$ to rewrite the symmetry relations in Eq.~(\ref{eq:d}) as: 
\begin{equation} \label{symmetry}
    \begin{split}
        b_1 = s - k_2, \quad b_3 = s + k_1, \\
        b_2 = s - k_1, \quad b_4 = s + k_2,
    \end{split}
\end{equation}
where $0 < k_1 \le k_2$, and it also gives $d = k_2 - k_1$. The piecewise form of step function approximator $f_s^{\text{NN}}$ in Eq.~(\ref{f_n_piece}) enables us to subdivide the error into parts like 
\begin{equation} \label{e_s}
    \begin{split}
        e_s^2 = & \int_{0}^1 \left| \gamma \left[ f_s^{\text{NN}}(x) + \frac{h}{2} \right] - \frac{h}{\Delta h} f_s(x) \right|^2 \dif x \\
        = & \, \gamma^2 \left[ \int_{b_1}^{s} \left| f_s^{\text{NN}}(x) + \frac{h}{2} \right|^2 \dif x + \int_{s}^{b_4} \left| f_s^{\text{NN}}(x) - \frac{h}{2} \right|^2 \dif x \right].
    \end{split}
\end{equation}
After calculation, the error of each approximator $f_s^{\text{NN}}$ can be estimated like
\begin{equation} \label{e_s^2}
    \begin{split}
        e_s^2 = \gamma^2 \left[ \frac{8}{3} (k_1 - k_2)^2 \left( k_1^3 + 3 k_1^2 k_2 + 2 k_1 k_2^2 + k_2^3 \right) \right] \le \gamma^2 \frac{56}{3} d^2 k_2^3.
    \end{split}
\end{equation}
During the construction of $f_s^{\text{NN}}$, the basis functions are chosen adjacently, leading to $k_2 \sim \mathcal{O}(d)$ and $e_s \sim \mathcal{O}(\gamma \, d^{\frac{5}{2}})$. To estimate the order of $\gamma$, we derive from Eq.~(\ref{eq:height}) that the height $h$ gives $h \sim \mathcal{O}(d^2)$, while the step function $f_s$ has a $d$-independent height $\Delta h \sim \mathcal{O}(1)$. Therefore, the scaling $\gamma = \Delta h / h \sim \mathcal{O}(d^{-2})$ is needed, and the error is rewritten as $e_s \sim \mathcal{O}(d^{\frac{1}{2}})$. Recall that $d$ in Eq.~(\ref{eq:d}) has $d \sim \mathcal{O}(\frac{1}{n-1})$, we have $e_s \sim \mathcal{O}(n^{-\frac{1}{2}})$, which means the approximation rate is roughly $1/2$ order with respect to the network width $n$. We will numerically verify this rate in Section~\ref{sec:experiments}.

This estimation also holds for our pseudo-copy $f_{s_j, p_l}^{\text{NN}}$, where $\gamma = L$. The triangle inequality is adopted to estimate the overall approximation error of these stacked $\{ f_{s, p_l}^{\text{NN}} \}_{l = 1}^L$ to a predetermined step function $f_s$, \emph{i.e.},
\begin{equation} \label{e_s_stacked}
    e_{s,p} = \left\| \sum_{l=1}^{L}  \left( f_{s, p_l}^{\text{NN}} + \frac{h}{2} \right) - f_s \right\|_{L^2} \le \sum_{l=1}^{L} \left\| \left( f_{s, p_l}^{\text{NN}} + \frac{h}{2M} \right) - \frac{1}{L} f_s \right\|_{L^2} =: \sum_{l=1}^{L} e_{s, p_l}.
\end{equation}
Now we focus on the approximation error $e_{s, p_l}$ of each $f_{s, p_l}^{\text{NN}}$ to the $f_s / L$. However, the result in Eq.~(\ref{e_s^2}) cannot be directly adopted since it only holds for the symmetry case in Eq.~(\ref{symmetry}).
Instead, we choose locations $\{ \tilde{b}_i \}_{i = 1}^4$ \emph{almost} symmetrically with mismatch measured by $\Delta s_l$. Therefore, the relation in Eq.~(\ref{symmetry}) is transformed into
\begin{equation} \label{b_mismatch}
	\begin{split}
        \tilde{b}_1 = (s + \Delta s_l) - k_2, \quad \tilde{b}_3 = (s + \Delta s_l) + k_1, \\ 
        \tilde{b}_2 = (s + \Delta s_l) - k_1, \quad \tilde{b}_4 = (s + \Delta s_l) + k_2.
    \end{split}
\end{equation}
Compared with Eq.~(\ref{symmetry}), it's clear that this transformation is equivalent to replacing $f_{s, p_l}^{\text{NN}}(x)$ with $f_{s}^{ \text{NN}}(x - \Delta s_l)$ for all $x \in [0,1]$. Therefore, each $e_{s, p_l}$ in Eq.~(\ref{e_s_stacked}) gives
\begin{equation} \label{e_stacked_each}
	\begin{split}
        e_{s, p_l}^2 = & \int_{0}^1 \left| \left[ f_{s_i}^{ \text{NN}}(x - \Delta s_l) + \frac{h}{2} \right] - \frac{1}{L} f_s(x) \right|^2 \dif x \\
        = & \int_{b_1 + \Delta s_l}^{s} \left| f_{s_i}^{\text{NN}}(x - \Delta s_l) + \frac{h}{2} \right|^2 \dif x + \int_{s}^{b_4 + \Delta s_l} \left| f_{s_i}^{\text{NN}}(x - \Delta s_l) - \frac{h}{2} \right|^2 \dif x,       
    \end{split}
\end{equation}
where the integral range $[b_1 + \Delta s_l, \, b_4 + \Delta s_l]$ is not symmetrical about $x = s$ due to the mismatch. However, since $\Delta s_l$ is small, we assume that $b_2 + \Delta s_l \le s \le b_3 + \Delta s_l$, then follow the similar procedure used in Eq.~(\ref{e_s}) to divide the error in Eq.~(\ref{e_stacked_each}) into parts. 
After calculation, we obtain the following estimation
\begin{equation*}
    \begin{split}
        e_{s, p_l}^2 
        = \, & \frac{8}{3} (k_1 - k_2)^2 \left[ k_1^3 + 3 k_1^2 k_2 + 2 k_1 k_2^2 + k_2^3 + 3 \Delta s_l^2 (k_1 + k_2) \right] \\
        = \, & \frac{8}{3} d^2 \left( - d^3 + 6 d^2 k_2 - 11 d k_2^2 + 7 k_2^3 - 3 d \, \Delta s_l^2 + 6 k_2 \, \Delta s_l^2 \right).
    \end{split}
\end{equation*}
Since $\Delta s_l$ can be assumed to be small as $\Delta s_l \sim \mathcal{O}(d)$, we obtain a similar estimation $e_{s, p_l} \sim \mathcal{O}(d^{\frac{5}{2}})$. Since the number of stacked pseudo-copy satisfies $L = \gamma \sim \mathcal{O}(d^{-2})$, the same estimation $e_{s,p} = L e_{s, p_l} \sim \mathcal{O}(n^{-\frac{1}{2}})$ is achieved. 

\subsection{Proof~of~Theorem \ref{th:random}} \label{sec:random}
Extending our results of permutation-trained networks to the random initialization scenarios imposes non-trivial challenges. Note that the symmetry construction of the step function approximator in Eq.~(\ref{four_pairs}) becomes invalid with random initializations, as the error introduced by randomness can accumulate as width increases.
Nevertheless, the randomly sampled parameters will become more dense upon increasing width, leading to a high probability of finding parameters that closely match the required location.

Therefore, we can first apply the UAP in the equidistant case to obtain a network $f^{\text{NN}}$ in Eq.~(\ref{NN}) (denoted as $f_{\text{equi}}^{\text{NN}}$ in the following text for distinction), which exhibits approximation power. Then, within a randomly initialized network $f^{\text{NN}}$ in Eq.~(\ref{NN}) (denoted as $f_{\text{rand}}^{\text{NN}}$) of sufficient width, we find a subnetwork $f_{\text{sub}}^{\text{NN}}$ that can be regarded as randomly perturbed from $f_{\text{equi}}^{\text{NN}}$. If this perturbation is small enough, the subnetwork $f_{\text{sub}}^{\text{NN}}$ will also possess approximation power.

\begin{proof}[Proof~of~Theorem \ref{th:random}]
We divide the whole proof into the following points:

\textbf{a)}  Approach $f^*$ with equidistantly initialized $f_{\text{equi}}^{\text{NN}}$.
Theorem \ref{th:main1} indicate that for any small $\varepsilon$ and the target function $f^*$, there is an equidistantly initialized network $f_{\text{equi}}^{\text{NN}}$ in Eq.~(\ref{NN}) with width $\tilde n$ and $B_{\text{equi}}^{(\tilde n)} = ( b_i )_{i = 1}^{\tilde n}$, $W_{\text{equi}}^{(2\tilde n)} = ( \pm b_i )_{i = 1}^{\tilde n}$, such that 
\begin{equation} \label{eq:f_equi}
    \left| f_{\text{equi}}^{\text{NN}}(x) - f^*(x) \right| < \varepsilon / 4, \quad \forall \, x \in [0,1].
\end{equation}

\textbf{b)}  Find a subnetwork $f_{\text{sub}}^{\text{NN}}$ in a randomly initialized $f_{\text{rand}}^{\text{NN}}$ to approximate $f_{\text{equi}}^{\text{NN}}$.
For a network with sufficiently wide $n \gg \tilde n$, parameters $B_{\text{rand}}^{(n)} \sim \mathcal{U} [0,1]^n$ and $W_{\text{rand}}^{(2n)} = ( \! \pm p_i )_{i = 1}^n$, $p_i \sim \mathcal{U} [0,1]$, we can find a subnetwork $f_{\text{sub}}^{\text{NN}}$ with parameters 
\begin{equation*}
    B_{\text{sub}}^{(\tilde n)} = \left( b_i + r_i^B \right)_{i = 1}^{\tilde n}, \quad W_{\text{sub}}^{(2 \tilde n)} = \big( \pm \left( b_i + r_i^W \right) \big)_{i = 1}^{2 \tilde n},
\end{equation*}
which can be viewed as randomly perturbed from $B_{\text{equi}}^{(\tilde n)}$, $W_{\text{equi}}^{(2\tilde n)}$, while the independent and identically distributed (i.i.d.) $r_i^B \sim \mathcal{U} [-\Delta r, \Delta r]^{\tilde n}$ and $r_i^W \sim \mathcal{U} [-\Delta r, \Delta r]^{2\tilde n}$ measured the perturbation, and $\Delta r > 0$ are the maximum allowable perturbation of $B_{\text{equi}}^{(\tilde n)}$ and $W_{\text{equi}}^{(2\tilde n)}$
we impose further constraints on the parameters near the boundary of $[0,1]$ to prevent them from exceeding the range).
Consequently, for sufficiently small $\Delta r < r_0$, the subnetwork $f_{\text{sub}}^{\text{NN}}$ will approach $f_{\text{equi}}^{\text{NN}}$ with the approximation error like
\begin{equation} \label{eq:f_sub}
    \left| f_{\text{sub}}^{\text{NN}}(x) - f_{\text{equi}}^{\text{NN}}(x) \right| < \varepsilon / 4, \quad \forall \, x \in [0,1], \quad \text{ with probability } P_{\text{sub}}.
\end{equation}
where $P_{\text{sub}}$ denote the related possibility.
This also enables $f_{\text{sub}}^{\text{NN}}$ the approximation power to $f^*$ due to its continuity with respect to the parameters. Combining the results in Eq.~(\ref{eq:f_equi})-(\ref{eq:f_sub}), the approximation error $E_{\text{rand}}^{\,\text{sub}}$ of the subnetwork gives that, with probability $P_{\text{sub}}$,
\begin{equation} \label{eq:E_sr}
    \begin{split}
        E_{\text{rand}}^{\,\text{sub}} := & \,
    \max_{x \in [0,1]}
    \left| f_{\text{sub}}^{\text{NN}}(x) - f^*(x) \right| \\
    \le & \, \max_{x \in [0,1]} \left| f_{\text{sub}}^{\text{NN}}(x) - f_{\text{equi}}^{\text{NN}}(x) \right| + \max_{x \in [0,1]} \left| f_{\text{equi}}^{\text{NN}}(x) - f^*(x) \right| < \varepsilon / 2.
    \end{split}
\end{equation}

\textbf{c)}  Estimate the probability $P_{\text{sub}}$ related to the approximation power of $f_{\text{sub}}^{\text{NN}}$.
Here we estimate the probability of finding such a subnetwork $f_{\text{sub}}^{\text{NN}}$ with the required approximation error $E_{\text{rand}}^{\,\text{sub}}$ in Eq.~(\ref{eq:E_sr}).
We start with the complement event $\mathcal{A}_k^c$: for a given location $\hat b_k \in B_\text{equi}^{(\tilde n)}$, there is no close enough locations in $B_{\text{rand}}^{(n)}$, \emph{i.e.}, falls in the interval $[\hat b_k - \Delta r, \hat b_k + \Delta r]$. 
Concretely, the probability has $\mathbb{P}[\mathcal{A}_k^c] = (1-2 \Delta r)^{n}$ since the interval length is $2 \Delta r$.
Hence, by choosing $\Delta r$ small enough such that these intervals have no overlap, \emph{i.e.}, $\Delta r < \min \{\delta_s, 1/2\tilde n\}$,
we apply the inclusion-exclusion principle \citep{feller1968introduction} to write the probability of finding all locations $B_\text{equi}^{(\tilde n)}$ in $B_{\text{rand}}^{(n)}$ as
\begin{equation*}
    \begin{split}
        & \mathbb{P} \left[ \,\bigcap_{k = 1}^{\tilde n} \mathcal{A}_k \,\right] = 1 - \mathbb{P} \left[ \, \bigcup_{k = 1}^{\tilde n} \mathcal{A}_k^c \, \right] = 1 - \sum_{k=1}^{\tilde n} (-1)^{k+1} \binom{\tilde n}{k} \left(1 - 2 k \Delta r \right)^n =: 1 - P'.
    \end{split}
\end{equation*}
where $\binom{\tilde n}{k}$ is the binomial coefficient, and $P'$ is the complement probability.
This probability also holds for finding $W_\text{equi}^{(2 \tilde n)}$ in pairwise $W_{\text{rand}}^{(2n)}$.
Consequently, we can find a subnetwork $f_{\text{sub}}^{\text{NN}}$ within $f_{\text{rand}}^{\text{NN}}$ that is close enough to $f_{\text{equi}}^{\text{NN}}$, such that the approximation error $E_{\text{rand}}^{\,\text{sub}}$ in Eq.~(\ref{eq:E_sr}) is achieved with the probability of 
\begin{equation} \label{eq:probability}
    P_{\text{sub}} = \mathbb{P} \big[ E_{\text{rand}}^{\,\text{sub}} < \varepsilon / 2 \big] = \left[ 1 - P' \right]^{2}.
\end{equation}
For given $\tilde n$ and $\Delta r$, we have $P' \to 0$ as $n \to \infty$. Therefore, we choose a sufficiently large $n$ to ensure that the probability $P_{\text{sub}} \ge \sqrt{1 - \delta}$.

\textbf{d)}  Annihilate the remaining part in the randomly initialized $f_{\text{rand}}^{\text{NN}}$.
We follow the same discussion as in the equidistant case to eliminate the unused parameters in $f_{\text{rand}}^{\text{NN}}$. 
By applying linear reorganization and Lemma~\ref{th:Leibniz}, we rewrite the remaining part as a linear function $\mathcal{S}_{\text{un}}^r(x) = \beta_r x + \eta_r$, 
where $x \in [0,1]$ and $0 \le \beta_r \le \Delta p$.
The upper bound $\Delta p$ is the largest gap between the adjacent coefficients $\{ p_i \}_{i \in I_{\text{un}}}$. 
As the network width $n$ increases, the randomly initialized $\{ p_i \}_{i \in I_{\text{un}}}$ become denser, leading to $\Delta p \xrightarrow{\text{a.s.}} 0$. Therefore, we can choose a sufficiently large $n$ to ensure $\Delta p \le \varepsilon / 2$ with high probability, \emph{i.e.}, $P_{\text{un}} \ge \sqrt{1-\delta}$, where $P_{\text{un}}$ denotes the related possibility.

Similarly to Eq.~(\ref{eq:minEqui}), we introduce an additional shift $C_r = -\eta_r$ to control the error of the intercept $\eta_r$. 
Therefore, the error $E_{\text{rand}}^{\,\text{un}}$ introduced by the unused part in $f_{\text{rand}}^{\text{NN}}$ satisfies can be estimated similarly with Eq.~(\ref{eq:minEqui}) as the following form,
\begin{equation} \label{eq:E_ur}
    E_{\text{rand}}^{\,\text{un}} = \max_{x \in [0,1]} \left| \mathcal{S}_{\text{un}}^r(x) + C_r \right| \le \Delta p \le \varepsilon / 2, \quad \text{ with probability } P_{\text{un}} > \sqrt{1-\delta}.
\end{equation}

\textbf{e)}  Complete the proof by combining the previous estimation. 
For any $\varepsilon > 0$ and $f^* \in C([0,1])$, we choose $\Delta r < \min \{r_0, \, \delta_s, \, 1/2 \tilde n \}$ and a large $n$ to ensure that 
    \begin{enumerate}
        \item With probability $P_{\text{sub}}$, the subnetwork $f_{\text{sub}}^{\text{NN}}$ can approximate the $f^*$ with the error $E_{\text{rand}}^{\,\text{sub}}$ satisfies $E_{\text{rand}}^{\,\text{sub}} < \varepsilon / 2 $ based on Eq.~(\ref{eq:E_sr});
        \item The probability $P_{\text{sub}}$ satisfies $P_{\text{sub}} \ge \sqrt{1 - \delta}$ based on Eq.~(\ref{eq:probability});
        \item With probability $P_{\text{un}} > \sqrt{1-\delta}$, the impact of the remaining parameters satisfies $E_{\text{rand}}^{\,\text{un}} < \varepsilon / 2$ based on Eq.~(\ref{eq:E_ur}).
    \end{enumerate}
    Hence, we can finish the proof by estimating the overall approximation error of our network $f_{\text{rand}}^{\text{NN}}$ to the target function $f^*$, which gives that with probability $1 - \delta$,
    \begin{equation} 
        \left| f_{\text{rand}}^{\text{NN}}(x) - f^*(x) \right|
        \le E_{\text{rand}}^{\,\text{sub}} + E_{\text{rand}}^{\,\text{un}} < \varepsilon, \quad \forall \, x \in [0,1].
    \end{equation}
\end{proof}

\section{Experiments}
\label{sec:experiments}

This section presents numerical evidence to support and validate the theoretical proof.
An interesting observation of permutation behaviors also highlights the theoretical potential of this method.

\subsection{The algorithmic implementation of permutation training}

In the implementation of permutation training, guidance is crucial in finding the weights' ideal order relationship.
\citet{qiu2020train} proposed \emph{lookahead permutation (LaPerm)} algorithm by introducing an $k$-times Adam-based free updating, where the learned relationship can then serve as a reference for permutation.
To ensure the performance, the weights are permuted after every $k$ epoch.
The impact of $k$'s value on convergence behavior is evaluated to be negligible (see \ref{app:k}).
Apart from the fixed permutation period $k$, it's also possible to adjust $k$ to learn sufficient information for the next permutation. Refer to \ref{app:algorithm} for more details about the LaPerm algorithm.

\subsection{Experimental settings}

To validate our theoretical results, we conduct experiments on regression problems. The settings are deliberately chosen to be consistent with our proof construction. We consider a three-layer network in Eq.~(\ref{NN}), where the first hidden layer's parameters are fixed to form the basis functions $\{\phi_i^\pm\}_{i = 1}^n$ in Eq.~(\ref{basis}). The weights $\theta^{(2n)}$ of the second hidden layer are trained by permutation, while the scaling factors $\alpha, \gamma$ in the output layer are freely trained to reduce the required network width. All the experiments below are repeated 10 times with different random seeds, and the error bars mark the range of the maximum and minimum values.
Refer to \ref{app:experiment} for the detailed experimental environment and setting.

\subsection{Approximating the one-dimensional continuous functions}

We utilize a $1$-$2n$-$1$-$1$ network architecture with equidistant initialization discussed in Theorem \ref{th:main1}, pairwise random initializations in Theorem \ref{th:random}, and also totally random initialization $W^{(2n)} \sim \mathcal{U}[-1, 1]^{2n}$.
The approximation targets are sine function $y = -\sin(2\pi x)$ and 3-order Legendre polynomial $y = \frac{1}{2} (5 x^3 - 3 x)$, $x \in [-1,1]$.

The numerical result illustrated in Fig.~\ref{fig:1D} exhibits a clear convergence behavior of equidistant and pairwise random cases upon increasing $n$, agreeing with our theoretical proof.
The clear dominance of pairwise random initialization indicates its undiscovered advantages in permutation training scenarios.
Besides, the total random case also shows a certain degree of approximation power.
However, in order to attain the equivalent accuracy, the total random case requires a wider network (\emph{i.e.}, a larger $n$).
Furthermore, the $L^\infty$ error exhibits a 1/2 convergence rate with respect to $n$. Although the theoretical estimation in Section~\ref{sec:proof} is based on $L^2$ norm, we indeed observe that it also holds for $L^\infty$ error.

\begin{figure*}[t]
    \begin{center}
    \begin{overpic}[height=4.6cm, clip=true,tics=10]{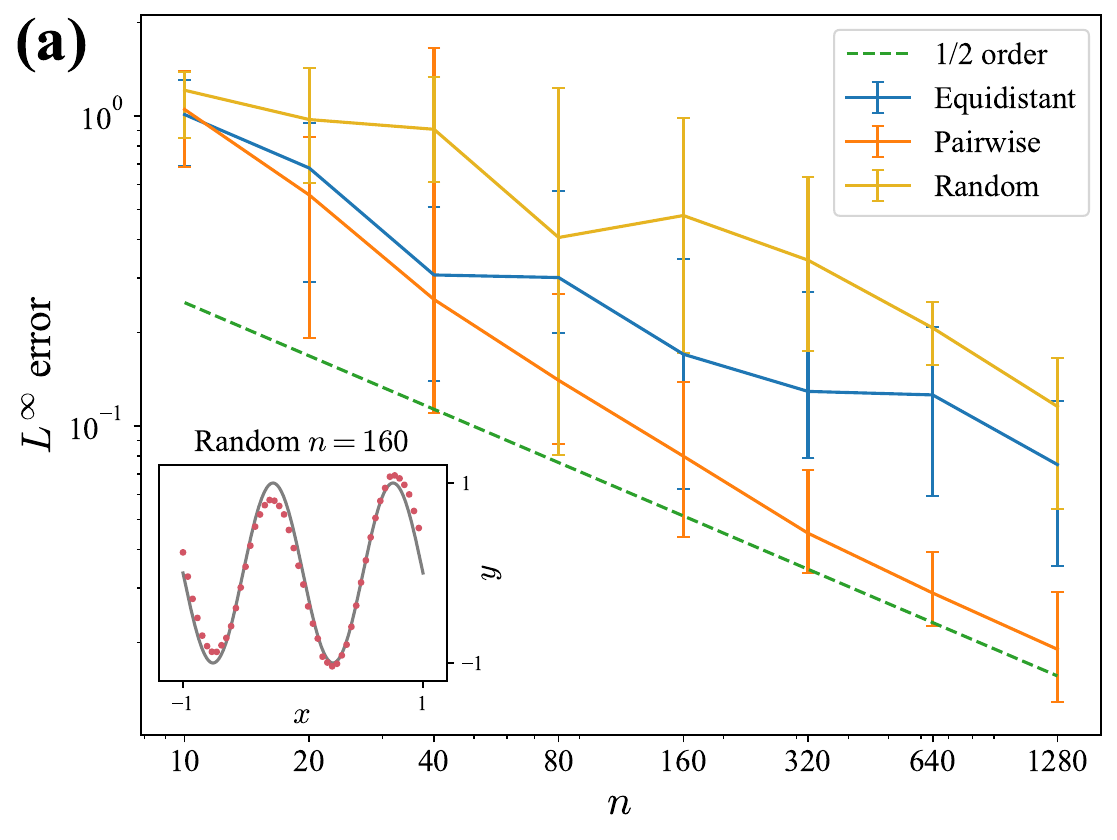}
    \end{overpic}\hspace{10pt}
    \begin{overpic}[height=4.6cm, clip=true,tics=10]{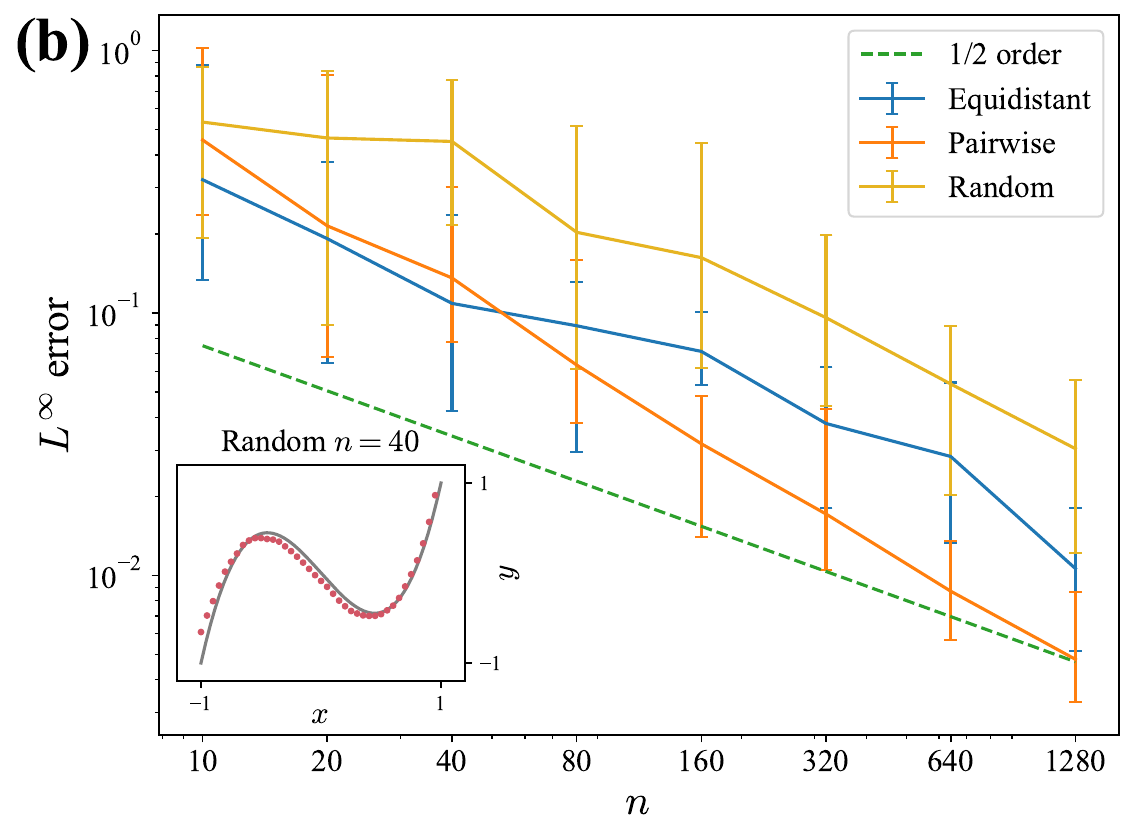}
    \end{overpic}
    \end{center}
    \caption{Approximating one-dimensional continuous function (a): $y = -\sin(2\pi x)$ and (b): $y = \frac{1}{2} (5 x^3 - 3 x)$ with equidistantly, pairwise random, and randomly initialized network, where $x \in [-1, 1]$. The inset in each panel presents the target function as lines and an example of the approximation result as dots.}
    \label{fig:1D}
\end{figure*}

\subsection{Approximating the multi-dimensional continuous functions} \label{app:2D}
As a natural extension, we consider a two-dimensional functions $z = - \sin \pi xy$, where $(x,y) \in [-1,1]^2$, starting with the construction of basis functions like $\phi_i^\pm (x)$ in Eq.~(\ref{basis}). Recall that in the one-dimensional case, the two subsets of basis $\phi_i^\pm (x)$ correspond the two opposite directions along the $x$-axis.
Therefore, at least four directions are required to represent a function defined on the $xy$-plane, of which two are parallel to the $x$-axis as $\phi_{i}^\pm (x, \cdot) = \text{ReLU} (\pm (x - b_i))$ and $\phi_{j}^\pm (\cdot, y) = \text{ReLU} (\pm (y - b_j))$ for $y$-axis, respectively.
Furthermore, inspired by the lattice Boltzmann method in fluid mechanics \citep{chen1998lattice}, we introduce another four directions as $\psi_{k}^{\pm \pm}(x,y) = \text{ReLU} (\pm x \pm y - b_k)$. So the whole basis functions are divided into eight subsets, each corresponding to a different direction (see Fig.~\ref{fig:2D}(b)).
Also, the range of biases is essential since the distribution of $\psi_{k}^{\pm \pm}(x,y)$ must be at least $\sqrt{2}$-times wider to cover the entire domain. Here we set the biases to range in varying directions with a uniform scaling factor, providing flexibility in dealing with the unwanted coefficients.

\begin{figure}[t]
    \begin{center}
    \begin{overpic}[height=5.2cm, clip=true,tics=10]{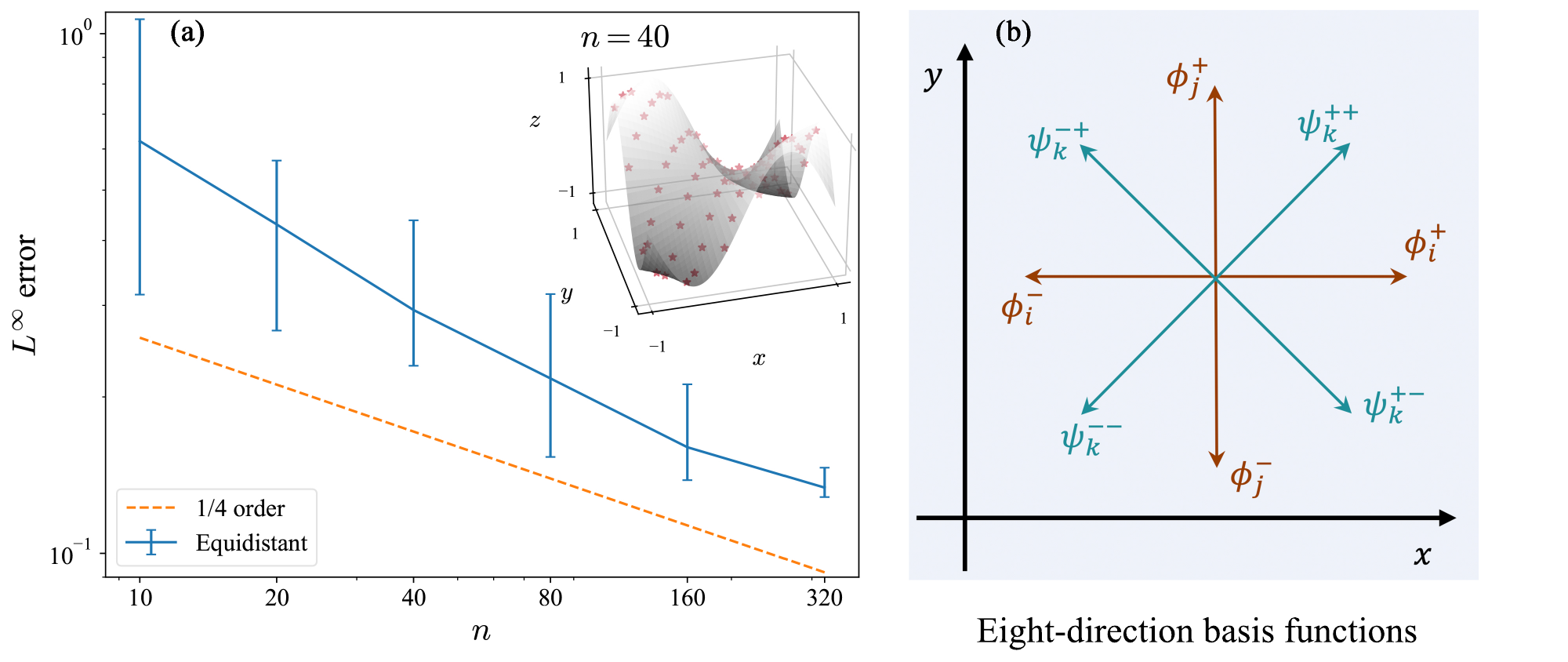}
    \end{overpic}
    \end{center}
    \caption{(a) Approximating two-dimensional continuous function $z = - \sin \pi xy$, where $x, y \in [-1, 1] \times [-1, 1]$. The inset panel presents the target function surface and an example of the approximation result as dots. (b) The two-dimensional basis function settings.}
    \label{fig:2D}
\end{figure}

Accordingly, we utilize a $2$-$8n$-$1$-$1$ network architecture and follow the same setting as before (refer to \ref{app:experiment}). The results depicted in Fig.~\ref{fig:2D}(a) also show good approximation power.
However, the $1/2$ convergence rate in previous cases cannot be attained here.
We hypothesize that this is due to our preliminary eight-direction setting of the basis functions.
This degeneration indicates the challenge of extending our theoretical results to higher dimensions. Further research will address this difficulty by considering more appropriate high-dimensional basis function setups. 
One possible approach relies on the basis construction utilized in the finite element methods \citep{brenner2008mathematical}. However, adopting such a method to meet the permutation training scenario is non-trivial and requires further investigation.

Moreover, the mismatch between the existing implementations and the permutation setting poses numerical challenges to permutation training in higher dimensions.
The performances are significantly affected by the algorithm implementations and initialization settings, both of which need further investigation and are beyond the scope of this work.
We hope our work can inspire and motivate the development of more sophisticated implementations specific to permutation training. However, the permutation training, as a numerical algorithm, can be directly applied to high-dimensional cases, even if it requires a significantly larger network width.

Using a similar numerical setting, we can also approximate functions with three-dimensional inputs. Here we consider $f(x,y,z) = \sin 3x \cdot \cos y \cdot \sin 2z$, where $(x,y,z) \in [-1,1]^3$. The results plotted in Fig.~\ref{fig:3D} demonstrate a certain degree of approximation power (due to the computational cost's limitation, we only conduct the experiments once). However, a degeneration convergence rate from $1/2$ to $1/6$ also indicates the theoretical limitations of the current construction.

\begin{figure}[t]
    \begin{center}
    \begin{overpic}[height=4cm, clip=true,tics=10]{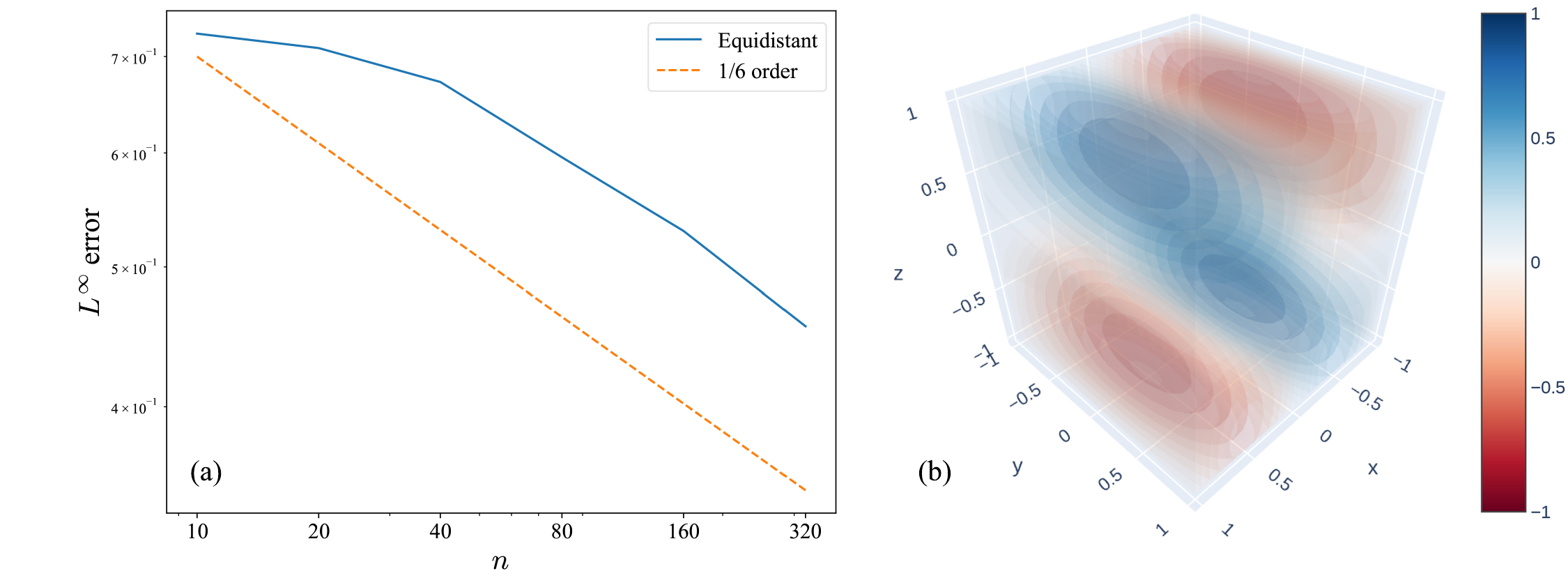}
    \end{overpic}
    \end{center}
    \caption{Approximating three-dimensional continuous function $f(x,y,z) = \sin 3x \cdot \cos y \cdot \sin 2z$, where $(x,y,z) \in [-1,1]^3$. (a) The convergence behavior under random seed 2022. (b) The three-dimensional illustration of the target function, where the function value $f(x,y,z)$ is plotted by the corresponding color in the color bar.}
    \label{fig:3D}
\end{figure}

\subsection{The influence of various initialization strategies} \label{sec:RandomExperiments}

Here we explore the effect of different initialization choices on the approximation behavior, which holds greater significance in permutation training scenarios due to the preservation of the initialized values.
The case in Fig.~\ref{fig:1D}(a) is utilized to apply various random initialization strategies.
The results plotted in Fig.~\ref{fig:initials} show that the UAP of permutation-trained networks is not limited in the setting considered by our previous theoretical investigation.
For more generalized cases, we first consider randomly initializing only $W^{(2n)}$ and $B^{(n)}$, which are labeled as Random 3 and 4, respectively. Both cases demonstrate competitive or even superior accuracy.

\begin{figure*}[t]
    \begin{center}
    \begin{overpic}[height=5.1cm, clip=true,tics=10]{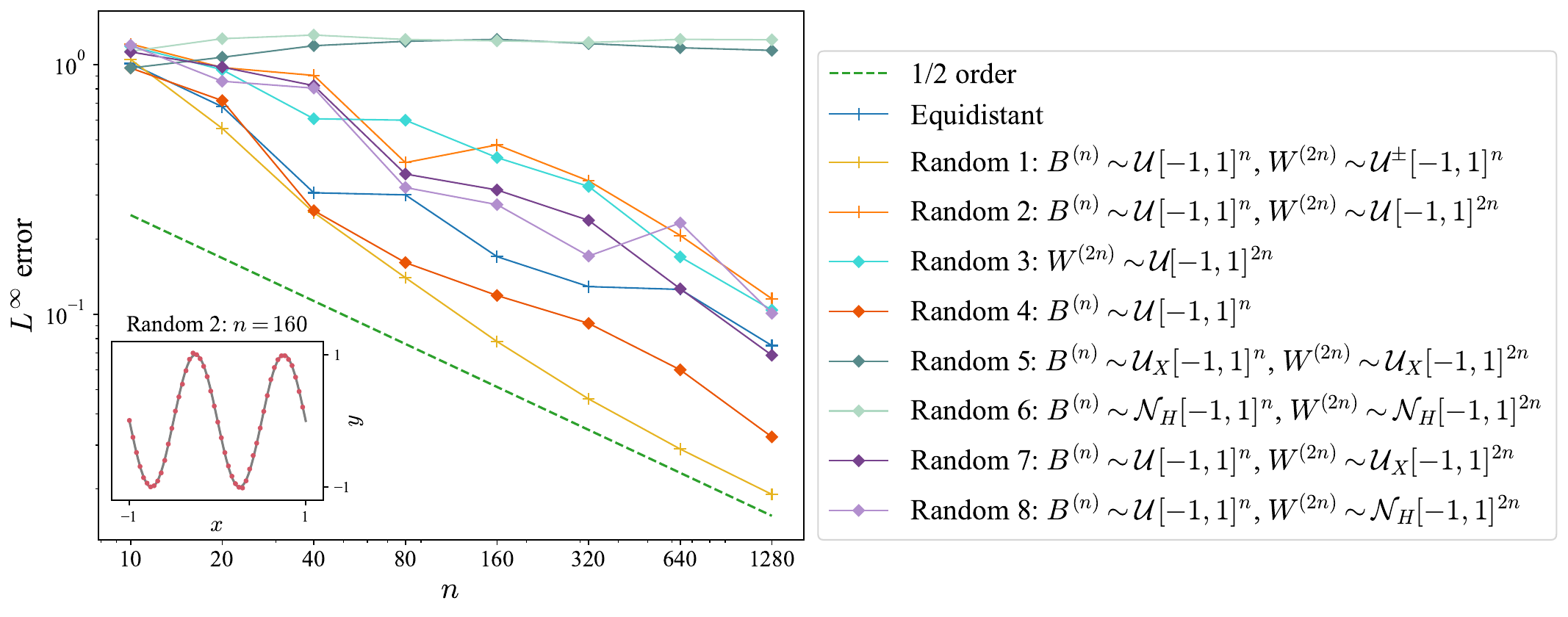}
    \end{overpic}
    \end{center}
    \caption{The performance of different initialization strategies in  approximating $y = -\sin(2\pi x)$ in $[-1, 1]$. The pairwise initialization $W^{(2n)} = ( \pm p_i )_{i = 1}^n, \, p_i \sim \mathcal{U}[-1,1]$ is denoted as $W^{(2n)} \sim \mathcal{U}^{\pm}[0,1]^n$. The error bars are omitted for conciseness. The inset panel presents the target function as lines and an example of the approximation result as dots.}
    \label{fig:initials}
\end{figure*}

Next, we consider some commonly used initializations, such as Xavier's uniform initialization $U_X$ \citep{glorot2010understanding}, and He's normal initialization $N_H$ \citep{he2015delving}. However, the implementation of $U_X$ on Random 5 and $N_H$ on Random 6 fails to result in convergence. This abnormal poor performance may be attributed to the mismatch of the scale in $B^{(n)}$ and the approximation domain $[0,1]$. To verify our hypothesis, we hold $B^{(n)} \sim \mathcal{U} [-1, 1]^n$, and only apply $U_X$ and $N_H$ on the coefficients $W^{(2n)}$ in Random 7 and 8. Consequently, both cases successfully regain the approximation ability.
These surprising behaviors, especially the unexpected deterioration of the default choices,
emphasizes the limited understanding of systematical characterizing the initialization suitable for permutation training scenarios.

\subsection{Observation of the permutation-active patterns} \label{sec:active}

This section aims to explore the theoretical potential of permutation training in describing network learning behavior.
Based on the significant correlation between permutation and learning behavior, as evidenced by \citet{qiu2020train} and our investigation, we hypothesize that the permutation-active components of the weights may play a crucial role in the training process.
Therefore, by identifying and tracing the permutation-active part of weights, a novel description tool can be achieved, which also facilitates visualization and statistical analysis of the learning behavior.

\begin{figure*}[t]
    \begin{center}
    \begin{overpic}[height=5.3cm, clip=true]{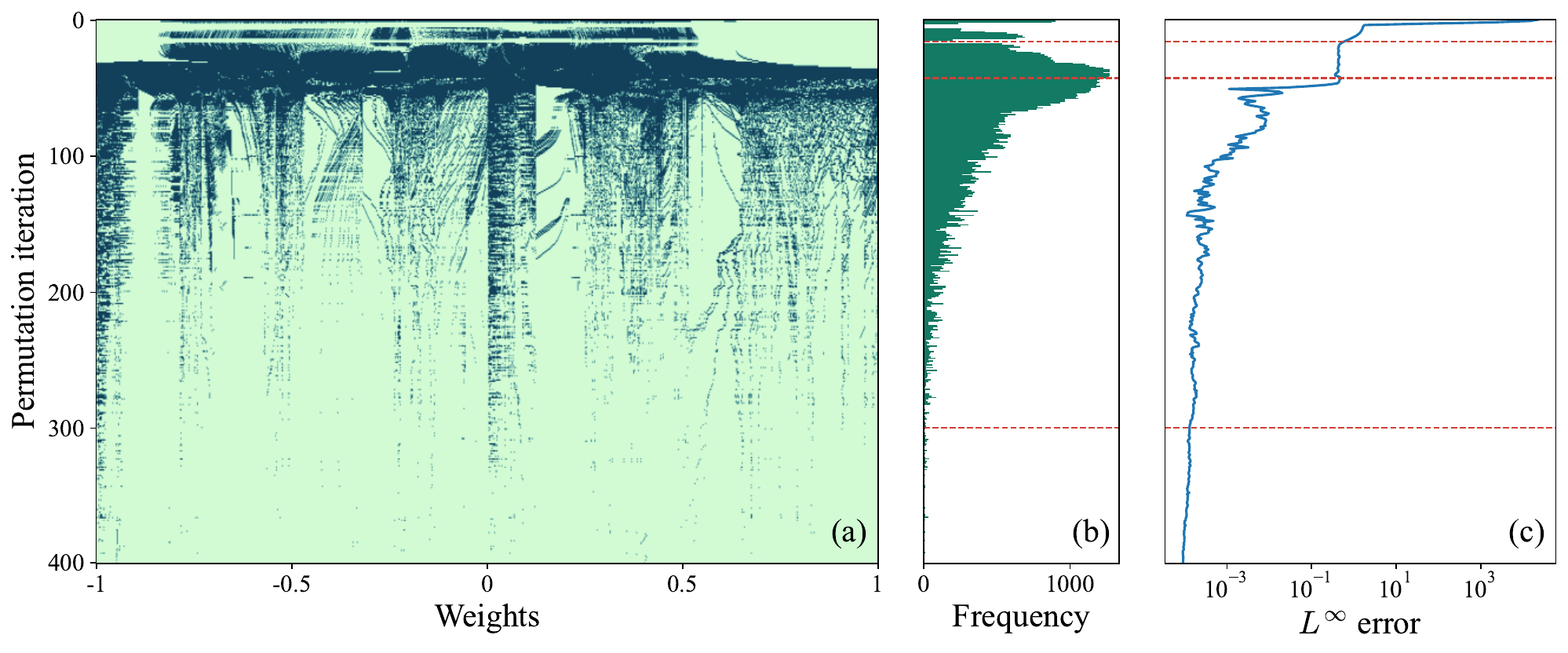}
    \end{overpic}
    \end{center}
    \caption{The permutation behavior in the first 400 permutation iteration in approximating $y = -\sin(2\pi x)$ by equidistantly initialized network with $n = 640$. (a) The distribution of the active components (denoted by dark green color).
    (b) The frequency distribution illustrates the variation in the total count of active components in each permutation.
    (c) The corresponding loss behavior. }
    \label{fig:IsPermu}
\end{figure*}

As a preliminary attempt, we illustrate the permutation behavior of the coefficients $\theta^{(2n)}$ in Fig.~\ref{fig:1D}(a). The components that participated in the permutation are visually highlighted in dark green in Fig.~\ref{fig:IsPermu}(a).
The behaviors plotted in Fig.~\ref{fig:IsPermu}(b)-(c) clearly show that the frequency of order relationship exchange evolves synchronously with the learning process, agreeing with the observation of \citet{qiu2020train}.

Specifically, the distribution of active components shows significant patterns, which are classified into four stages (marked by red dash lines in Fig.~\ref{fig:IsPermu}).
The loss declines sharply in the initial stage, while only the components with medium value are permuted.
Once loss reaches a plateau in the second stage, more components are involved in permutation, evidencing the role of permutation in propelling the training.
As loss starts to decline again, the permutation frequency correspondingly diminishes.
Interestingly, the slower loss decrease gives rise to a ribbon-like pattern, akin to the reported localized permutations (plotted in Fig.~14 of the appendix in \citet{qiu2020train}), and possibly due to slow updates failing to trigger a permutation.
This observation may support the existence of inherent low-dimensional structures within the permutation training dynamics, potentially linked to mathematical depiction of permutation groups, such as cycle decomposition \citep{cameron1999permutation} and Fourier bases for permutation \citep{huang2009fourier}.
Finally, the permutation's saturation aligns with the stationary state of loss convergence.
We believe these inspiring phenomena deserve further exploration, as they hold the potential to provide novel insights regarding the learning behavior of networks.

\section{Discussion and Future Work}
\label{sec:discussion}

Despite the exploratory nature as the first theoretical work (to our knowledge), our findings suggest the valuable potential of permutation training. Some intriguing questions are also raised for future research.

\subsection{Generalizing to other scenarios}
Although we mainly focus on basic settings, the proof idea exhibits generalizability.
Extending to networks equipped with leaky-ReLU can be straightforward (refer to \ref{app:leakyReLU} for numerical evidence). 
Our proof also enables implementations within other designs, such as networks with deeper architectures or sparse initializations (see \ref{app:extending} for detailed discussion).

However, extending our theoretical results to the high-dimensional scenario still faces challenges.
One of the primary obstacles is constructing multi-dimensional basis functions that are suitable for the permutation training scenario.
A reasonable approach relies on the construction in the finite element methods \citep{brenner2008mathematical}. 
We plan to address this problem in future work. 
As for the numerical evidence of high-dimensional scenarios, \citet{qiu2020train} have examined the classification problem using VGG-based networks on the CIFAR-10 dataset, while our experiments in regression task have shown approximation behavior for two and three-dimensional inputs (see Section~\ref{app:2D}).

\subsection{Permutation training as a theoretical tool}

Our observation in Sec.~\ref{sec:active} indicates the theoretical potential of permutation training, as it corresponds well with the training process and has systematical mathematical descriptions.
Specifically, the patterns observed in Fig. \ref{fig:IsPermu} can intuitively lead to certain weight categorization strategies, potentially benefit consolidating the crucial weights for previous tasks \citep{maltoni2019continuous}, or pruning to find the ideal subnetwork in the lottery ticket hypothesis \citep{frankle2019lottery}.
Additionally, the existing permutation training algorithm shares the same form as weight projection methods in continual learning \citep{zeng2019continual}, as it can be viewed as applying an order-preserving projection from the free training results to the initial weight value.

\subsection{The algorithmic implementation}

This work is expected to facilitate the applications of permutation training. However, some practical issues still exist and deserve further investigation. 
As a preliminary attempt, the existing permutation training algorithm LaPerm guides the permutation by inner loops, thus incurring undesirable external computation costs. However, employing more advanced and efficient search approaches, such as the learn-to-rank formalism \citep{cao2007learning}, or permutation search algorithms in the study of LMC \citep{jordan2023repair, ainsworth2023git}, the benefits of permutation training will be actualized in practice. 
Importantly, our proof does not rely on any specific algorithmic implementations.
Additionally, the incompatible initialization issue plotted in Fig.~\ref{fig:1D}(b) emphasizes the need for developing more effective initializations as well as investigating the criterion of UAP-compatible initializations.

\section{Conclusion}
\label{sec:conclusion}
As a novel method, permutation training exhibits unique properties and practical potential.
To verify its efficacy, we prove the UAP of permutation-trained networks with random initialization for one-dimensional continuous functions. The proof is generalized from the equidistant scenario, where the key idea involves a four-pair construction of step function approximators in Fig.~\ref{fig:main}, along with a processing method to eliminate the impact of the remaining parameters. 
Our numerical experiments not only confirm the theoretical results through Fig.~\ref{fig:1D}(a), but also validate the prevalence of the UAP of permutation-trained networks in various initializations in Fig.~\ref{fig:1D}(b).
The discovery that commonly used initializations fail to achieve UAP also raises an intriguing question about the systematic characterization of initializations that satisfy UAP. Our observation in Fig.~\ref{fig:IsPermu} suggests that permutation training could be a novel tool to describe the network learning behavior.

Overall, we believe that the UAP of permutation-trained networks reveals the profound, yet untapped insights into how the weight encodes the learned information, highlighting the importance of further theoretical and practical exploration.

\section*{CRediT authorship contribution statement}

\textbf{Y. Cai:} Conceptualization, Validation, Formal analysis, Writing - original draft, Writing - review \& editing, Visualization, Project administration.
\textbf{G. Chen:} Conceptualization, Methodology, Software, Validation, Formal analysis, Investigation, Data curation, Writing - original draft, Writing - review \& editing, Visualization.
\textbf{Z. Qiao:} Validation, Resources, Writing - review \& editing, Supervision, Project administration, Funding acquisition.

\section*{Declaration of competing interest}

The authors declare that they have no known competing financial interests or personal relationships that could have appeared to influence the work reported in this paper.

\section*{Data availability}

The code and data accompanying this manuscript are publicly available on GitHub at \uline{https://github.com/DanclaChen/PermutationTraining}.

\section*{Acknowledgments}

This work is supported by the CAS AMSS-PolyU Joint Laboratory of Applied Mathematics.
The work of Y. Cai is supported by the National Natural Science Foundation of China under Grant 12201053 and 12171043.
The work of Z. Qiao is supported by the Hong Kong Research Grants Council
RFS grant RFS2021-5S03 and GRF grant 15302122, and the Hong Kong Polytechnic University grant 4-ZZLS.

\appendix

\section{The experimental setting} \label{app:experiment}
To establish the convergence property upon increasing network width, we sample the training points randomly and uniformly in $[-1,1]$, along with equidistantly distributed test points. The maximum training epoch is sufficiently large to ensure reaching the stable state. For the multi-dimensional case, we set the basis functions at a larger domain than the functions to ensure accuracy near the boundary. The scale is measured by $T_b$, which means the biases are in $[-1-T_b, 1+T_b]$ in each dimension. See Table~\ref{table:hyper} for detailed choice.

\begin{table}[htp!]
  \caption{Hyperparameters setting.}
  \label{table:hyper}
  \centering
  \begin{tabular}{lccc}
    \toprule
    Hyperparameters     & 1D & 2D & 3D  \\
    \midrule
    Architectures & $1$-$2n$-$1$-$1$ & $2$-$8n$-$1$-$1$  & $3$-$26n$-$1$-$1$  \\
    $k$ & 5 & 5 & 20 \\
    Batch size & 8 & 128 & 640  \\
    \# training points & 1600 & \multicolumn{2}{c}{51200}  \\
    \# test points & 400 & \multicolumn{2}{c}{12800} \\
    $T_b$ & 0 & \multicolumn{2}{c}{0.75} \\
    \midrule
    $n$ & \multicolumn{3}{c}{$\{ 10,20,40,80,160,320 \}$} \\
    \# epoch & \multicolumn{3}{c}{6400} \\
    Learning rate (LR) & \multicolumn{3}{c}{1e-3} \\
    Multiplicative factor of LR decay & \multicolumn{3}{c}{0.998} \\
    Multiplicative factor of $k$ increase & \multicolumn{3}{c}{$\sqrt[10]{1.002}$} \\

    \bottomrule
  \end{tabular}
\end{table}

The experiments are conducted in NVIDIA A100 Tensor Core GPU. However, our code is hardware-friendly since each case only consumes approximately 2GB of memory. The code uses \emph{torch.multiprocessing} in PyTorch 2.0.0 with ten different random seeds, namely $2022, 3022, \cdots, 12022$. Additionally, the training data of each case is sampled under the random seed $2022$ to ensure that they are comparable.

\section{Adam-based LaPerm algorithm} \label{app:algorithm}
Here we briefly introduce the Adam-based LaPerm algorithm with a fixed permutation period $k$ in \citet{qiu2020train}. The pseudocode is shown in Algorithm \ref{alg:LaPerm}.
\begin{algorithm}[htp]
    \caption{Adam-based LaPerm algorithm}
    \label{alg:LaPerm}
    \begin{algorithmic}
    \REQUIRE{Loss function $\mathcal{L}$, training set $D_T$, maximum training epoch $M_e$}
    \REQUIRE{Inner optimizer Adam, permutation period $k$, initial weights $W$}
    \STATE $\theta_0 = W \quad $   // Initialize the weights
    \FOR{$ t = 1,2,\ldots, M_e$}
        \STATE $\theta_{t} \leftarrow \text{Adam}(\mathcal{L}, \theta_{t-1}, D_T) \quad$   // Free training by Adam
        \IF{$k \text{ divides } t$}
            \STATE $\theta_{t} \leftarrow \tau_t(W) \quad$ // Apply the permutation
        \ENDIF
    \ENDFOR
    \end{algorithmic}
\end{algorithm}
This algorithm rearranges the initial weights $W$ guided by the order relationship of $\theta_{t}$, so the trained weights will hold the initial value. Therefore, it can be regarded as a permutation of the initial weights $W$.

\section{The impact of permutation period on convergence behavior} \label{app:k}
As a hyperparameter, the choice of permutation period $k$ during the implementation of LaPerm algorithms has the possibility to affect the convergence behavior. The correlation between the value of $k$ and the final accuracy is reported to be unambiguous (refer to Fig.~6 in \citet{qiu2020train}).
Generally, a larger $k$ is associated with slightly higher accuracy of single permutation training result, thus in our experiments, the weights are permuted after each $k$ epoch.
Fig. \ref{fig:k} evaluates the impact of $k$'s value on convergence behavior, whose results suggest that this effect remains negligible.

\begin{figure}[t!]
    \centering
    \includegraphics[height = 5cm]{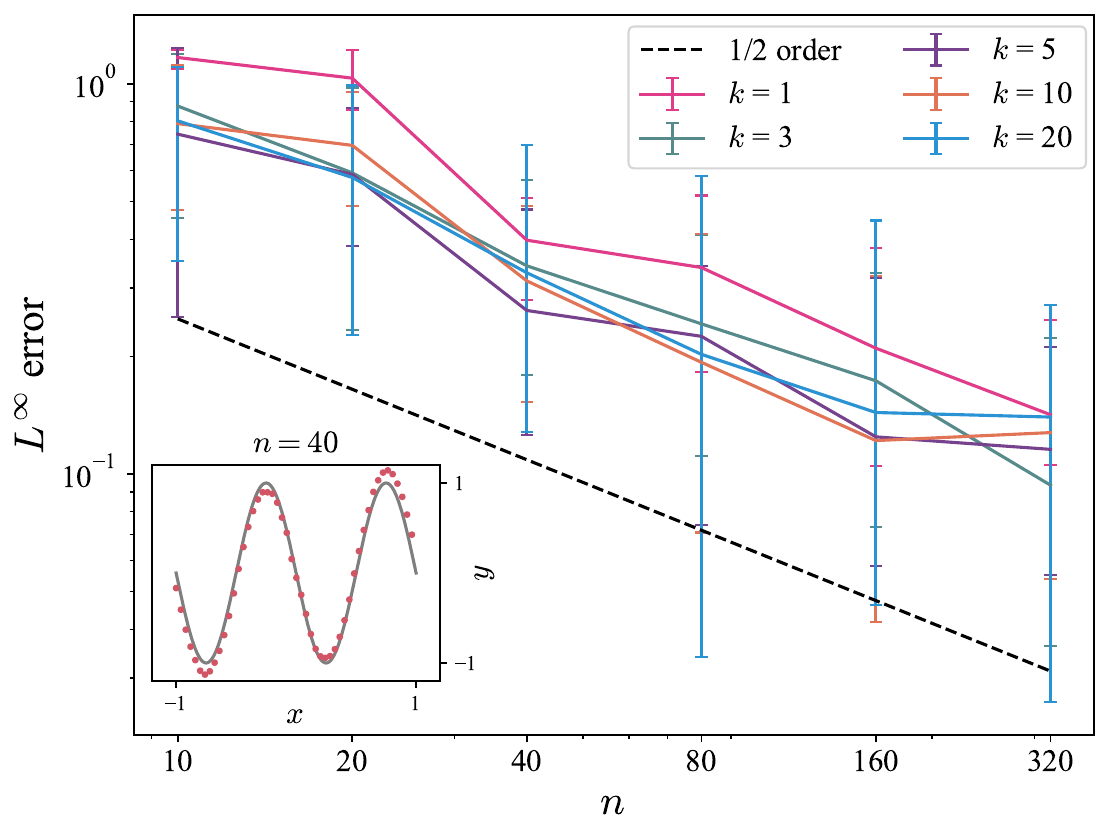}
    \caption{Approximating one-dimensional continuous function $y = a_0 + a_1 \sin(\pi x) + a_2 \cos(2 \pi x)+ a_3 \sin(3 \pi x)$ with equidistantly initialized network, where $x \in [-1, 1]$, and the value of permutation period $k = 1, 3, 5, 10, 20$, respectively. The inset in each panel presents the target function as lines and an example of the approximation result as dots.}
    \label{fig:k}
\end{figure}

\section{Generalize to the networks equipped with leaky-ReLU} \label{app:leakyReLU}
Extending our outcomes to leaky-ReLU is expected. This is because of the two crucial techniques deployed in our proof: constructing the step function approximators and eliminating the unused parameters, both can be applied to the leaky-ReLU.

Since our two-direction setting of basis function $\phi^\pm$ in Eq.~(\ref{basis}) can impart leaky-ReLU with symmetry equivalent to ReLU, it's feasible to construct a similar step function approximator by rederiving the relevant coefficient $p_i, q_i$. Furthermore, the existing eliminating method can be directly employed since a pair of leaky-ReLU basis functions can be constructed into linear functions for further processing.

As an initial attempt, we numerically examine the leaky-ReLU networks by only changing the activation function in the case of Fig.~\ref{fig:1D}(a). The results plotted in Fig.~\ref{fig:leakyReLU} exhibit the approximation power of leaky-ReLU networks. Unlike the ReLU cases, the random initialization outperforms the equidistant initialization. However, the $1/2$ convergence rate in previous ReLU cases cannot be attained here, probably because the proof based on leaky-ReLU may result in different constants, leading to potential discrepancies in details
when compared with the ReLU-based conclusions.

\begin{figure}[t]
    \centering
    \includegraphics[height = 5cm]{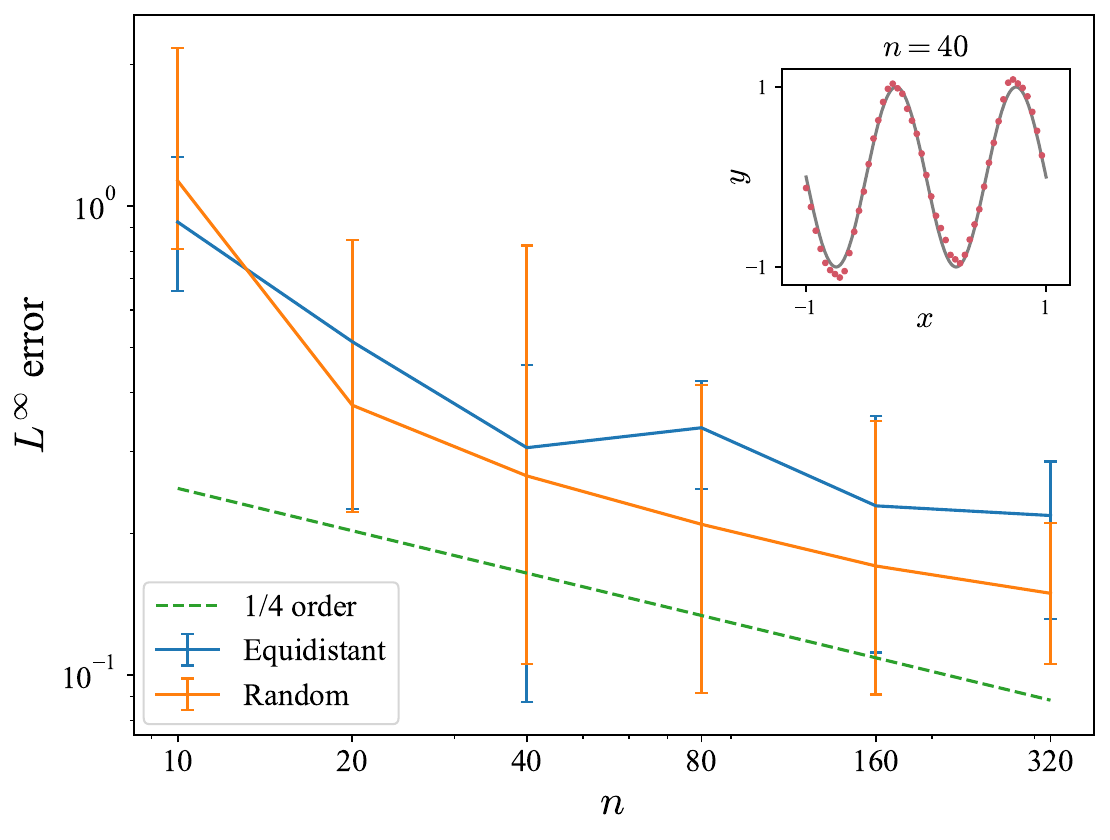}
    \caption{Approximating one-dimensional continuous function $y = -\sin(2\pi x)$ with equidistantly initialized network equipped with leaky-ReLU, where $x \in [-1, 1]$. The inset in each panel presents the target function as lines and an example of the approximation result as dots.}
    \label{fig:leakyReLU}
\end{figure}

\section{Extending to the networks with deeper architecture or sparse initializations} \label{app:extending}
This paper follows the conventional setting of free training UAP study and primarily focuses on shallow fully connected networks. This is based on the premise that the conclusions can be readily extended to scenarios involving deeper networks.
As for deep networks, the permutation training setting seems to inhibit a direct extension. However, our linear reorganization in Eq.~(\ref{eq:unused_linear}) enables construct an identity function $y = x$ using a pair of basis functions $y = p_n \phi_1^+(x) +q_n \phi_1^-(x)$, where $b_1 = 0$, $p_n = 1$, $q_n = -1$. This process enables us to utilize identity functions within subsequent layers. Consequently, the deep networks scenario parallels the cases discussed in this paper, which allows us to realize UAP within deep networks.

Additionally, this paper's conclusions could be extended to sparse initialization scenarios, an area where permutation training shows significant advantages (see Fig. 6(e) in \citet{qiu2020train}). 
On an intuitive level, our proof can explain these advantages, as a primary challenge of permutation training is managing unused basis functions. However, in sparse situations, we can conveniently deal with this problem by assigning the remaining basis functions zero coefficients, thereby facilitating the proof process.

\section{A toy example of our step-matching and constant-matching construction} \label{app:toy}
To illustrate the key idea of our construction, we consider a toy example of approaching a step function $f_s(x) = 0.8 \, \chi(x - 0.4)$ for $x \in [0,1]$ by a network in Eq.~(\ref{NN}) with $n = 11$ and equidistantly initialized $\{ b_k \}_{k = 1}^{11} = \{ 0, 0.1, \cdots, 1 \}$. This setting is considered after applying Lemma 2.1 with $\Delta h = 0.8$.

Following our step-matching construction, we can choose the basis functions with $\{ b_k \}_{k = 1}^4 = \{ 0.1, 0.3, 0.6, 0.8 \}$ to be step-matching in Eq.~(\ref{eq:coefficients_step}), leading to a step function approximator $f_s^{\text{NN}} $. The height of $f_s^{\text{NN}}$ is given by $h = 0.4$. Therefore, the target step function $f_s$ can be approximated by $f_s^{\text{NN}}$ along with shifting $\Delta h/2 = 0.4$ and scaling $\Delta h / h = 2$ with the following error estimation:
\begin{equation*}
\left| \left[ \frac{\Delta h}{h}(x) \, f_s^{\text{NN}} + \frac{h}{2} \right] - f_s(x) \right| \le 0.8 = \Delta h, \quad x \in [0,1].
\end{equation*}
The target step function $f_s$ and its approximator $f_s^{\text{NN}}$ are plotted in Fig.~\ref{fig:toyStep}. 

\begin{figure}
    \centering
    \includegraphics[width=0.5\textwidth]{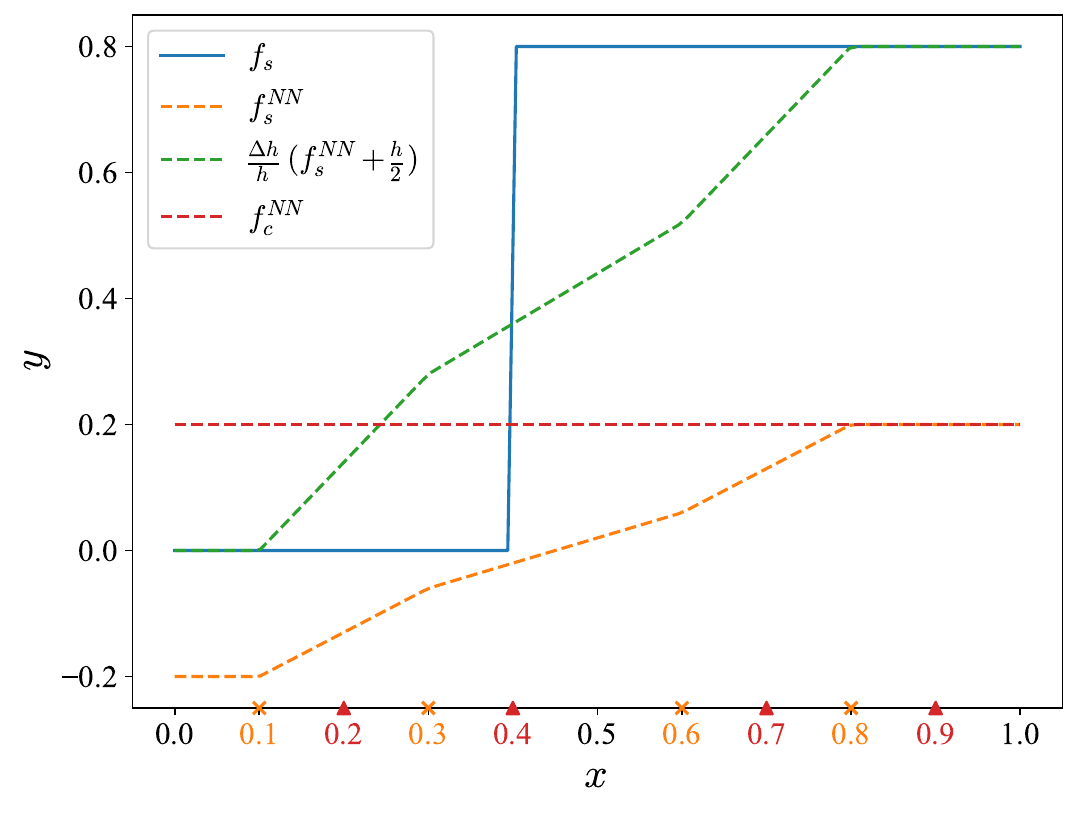}
    \caption{Approximating a step function $f_s(x) = 0.8 \, \chi(x - 0.4)$ and a constant function $f_c(x) = 0.2$ in $x \in [0, 1]$ with the step-function approximator $f_s^{\text{NN}}$ and the constant function approximator $f_c^{\text{NN}}$, respectively. The target functions are plotted as lines, and the approximation results are shown as dashed lines. The locations of the basis functions used for $f_s^{\text{NN}}$ and $f_c^{\text{NN}}$ are marked with "$\times$" and "$\blacktriangle$" on the $x$-axis, respectively. }
    \label{fig:toyStep}
\end{figure}

Additionally, we can consider the constant function approximator $f_c^{\text{NN}}$ following our constant-matching construction in Eq.~(\ref{eq:coefficients_constant}). We choose the basis functions located at $\{ 0.2, 0.4, 0.7, 0.9 \}$ to be constant-matching, leading to a constant function approximator $f_c^{\text{NN}} = h/2$. The height of $f_c^{\text{NN}}$ is given by $h/2 = 0.4$. From the shape of $f_c^{\text{NN}}$ plotted in Fig.~\ref{fig:toyStep}, we can see that it can approximate a constant function without any error.

\section{A toy example of Lemma~\ref{th:Leibniz}}   \label{app:Leibniz}

To illustrate the processing method of Lemma~\ref{th:Leibniz}, we consider a given sequence 
\begin{equation*}
 3, \, 1, \, 4, \, -1, \, 5, \, 9,
\end{equation*}
with the largest gap $\Delta c = 4$. The aim is to choose choose $m_i \in \{-1, 1 \}$ to reduce the absolute value of the result
\begin{equation*}
S = 3 \cdot m_1 + 1 \cdot m_2 + 4 \cdot m_3 - 1 \cdot m_4 + 5 \cdot m_5 + 9 \cdot m_6.
\end{equation*}
This can be achieved by the following steps:
\begin{itemize}
    \item[1.] Rearrange the sequence into $9, \, 5, \, 4, \, 3, \, 1, \, -1$;
    \item[2.] Compute the gap between every two adjacent elements as $4, \, 1, \, 2$;
    \item[3.] Reorganize the gap sequence into an alternating series and compute the result as $S = 4 - 2 + 1 = 3$;
    \item[4.] Transfer the choice of $m_i \in \{-1, 1 \}$ to the original sequence as 
    \begin{equation*}
        S = (9 - 5) - [1 - (-1)] + (4 - 3) = 9 - 5 - 1 - 1 + 4 - 3 = 3 < 4 = \Delta c.
    \end{equation*}
\end{itemize}
Therefore, the result $S = 3$ can be achieved by choosing $m_1 = -1, m_2 = -1, m_3 = 1, m_4 = 1, m_5 = -1, m_6 = 1$.

\bibliographystyle{elsarticle-harv} 
\bibliography{refs}

\begin{thebibliography}{46}
\expandafter\ifx\csname natexlab\endcsname\relax\def\natexlab#1{#1}\fi
\providecommand{\url}[1]{\texttt{#1}}
\providecommand{\href}[2]{#2}
\providecommand{\path}[1]{#1}
\providecommand{\DOIprefix}{doi:}
\providecommand{\ArXivprefix}{arXiv:}
\providecommand{\URLprefix}{URL: }
\providecommand{\Pubmedprefix}{pmid:}
\providecommand{\doi}[1]{\href{http://dx.doi.org/#1}{\path{#1}}}
\providecommand{\Pubmed}[1]{\href{pmid:#1}{\path{#1}}}
\providecommand{\bibinfo}[2]{#2}
\ifx\xfnm\relax \def\xfnm[#1]{\unskip,\space#1}\fi
\bibitem[{Ainsworth et~al.(2023)Ainsworth, Hayase and Srinivasa}]{ainsworth2023git}
\bibinfo{author}{Ainsworth, S.}, \bibinfo{author}{Hayase, J.}, \bibinfo{author}{Srinivasa, S.}, \bibinfo{year}{2023}.
\newblock \bibinfo{title}{{Git Re-Basin}: Merging models modulo permutation symmetries}, in: \bibinfo{booktitle}{The Eleventh International Conference on Learning Representations}.
\bibitem[{Brenner and Scott(2008)}]{brenner2008mathematical}
\bibinfo{author}{Brenner, S.C.}, \bibinfo{author}{Scott, L.R.}, \bibinfo{year}{2008}.
\newblock \bibinfo{title}{The mathematical theory of finite element methods}.
\newblock \bibinfo{publisher}{Springer}.
\bibitem[{Cai(2023)}]{cai2023achieve}
\bibinfo{author}{Cai, Y.}, \bibinfo{year}{2023}.
\newblock \bibinfo{title}{Achieve the minimum width of neural networks for universal approximation}, in: \bibinfo{booktitle}{The Eleventh International Conference on Learning Representations}.
\bibitem[{Cai et~al.(2021)Cai, Chen and Liu}]{cai2021least}
\bibinfo{author}{Cai, Z.}, \bibinfo{author}{Chen, J.}, \bibinfo{author}{Liu, M.}, \bibinfo{year}{2021}.
\newblock \bibinfo{title}{Least-squares {ReLU} neural network ({LSNN}) method for linear advection-reaction equation}.
\newblock \bibinfo{journal}{Journal of Computational Physics} \bibinfo{volume}{443}, \bibinfo{pages}{110514}.
\bibitem[{Cameron(1999)}]{cameron1999permutation}
\bibinfo{author}{Cameron, P.J.}, \bibinfo{year}{1999}.
\newblock \bibinfo{title}{Permutation groups}.
\newblock \bibinfo{number}{45}, \bibinfo{publisher}{Cambridge University Press}.
\bibitem[{Cao et~al.(2007)Cao, Qin, Liu, Tsai and Li}]{cao2007learning}
\bibinfo{author}{Cao, Z.}, \bibinfo{author}{Qin, T.}, \bibinfo{author}{Liu, T.Y.}, \bibinfo{author}{Tsai, M.F.}, \bibinfo{author}{Li, H.}, \bibinfo{year}{2007}.
\newblock \bibinfo{title}{Learning to rank: from pairwise approach to listwise approach}, in: \bibinfo{booktitle}{Proceedings of the 24th international conference on Machine learning}, pp. \bibinfo{pages}{129--136}.
\bibitem[{Castro et~al.(2000)Castro, Mantas and Ben{\i}tez}]{castro2000neural}
\bibinfo{author}{Castro, J.L.}, \bibinfo{author}{Mantas, C.J.}, \bibinfo{author}{Ben{\i}tez, J.}, \bibinfo{year}{2000}.
\newblock \bibinfo{title}{Neural networks with a continuous squashing function in the output are universal approximators}.
\newblock \bibinfo{journal}{Neural Networks} \bibinfo{volume}{13}, \bibinfo{pages}{561--563}.
\bibitem[{Chen and Doolen(1998)}]{chen1998lattice}
\bibinfo{author}{Chen, S.}, \bibinfo{author}{Doolen, G.D.}, \bibinfo{year}{1998}.
\newblock \bibinfo{title}{Lattice boltzmann method for fluid flows}.
\newblock \bibinfo{journal}{Annual review of fluid mechanics} \bibinfo{volume}{30}, \bibinfo{pages}{329--364}.
\bibitem[{Cohen and Welling(2016)}]{cohen2016group}
\bibinfo{author}{Cohen, T.}, \bibinfo{author}{Welling, M.}, \bibinfo{year}{2016}.
\newblock \bibinfo{title}{Group equivariant convolutional networks}, in: \bibinfo{booktitle}{International conference on machine learning}, \bibinfo{organization}{PMLR}. pp. \bibinfo{pages}{2990--2999}.
\bibitem[{Cybenko(1989)}]{cybenko1989approximation}
\bibinfo{author}{Cybenko, G.}, \bibinfo{year}{1989}.
\newblock \bibinfo{title}{Approximation by superpositions of a sigmoidal function}.
\newblock \bibinfo{journal}{Mathematics of control, signals and systems} \bibinfo{volume}{2}, \bibinfo{pages}{303--314}.
\bibitem[{Entezari et~al.(2021)Entezari, Sedghi, Saukh and Neyshabur}]{entezari2021role}
\bibinfo{author}{Entezari, R.}, \bibinfo{author}{Sedghi, H.}, \bibinfo{author}{Saukh, O.}, \bibinfo{author}{Neyshabur, B.}, \bibinfo{year}{2021}.
\newblock \bibinfo{title}{The role of permutation invariance in linear mode connectivity of neural networks}, in: \bibinfo{booktitle}{International Conference on Learning Representations}.
\bibitem[{Fan et~al.(2020)Fan, Xiong and Wang}]{fan2020universal}
\bibinfo{author}{Fan, F.}, \bibinfo{author}{Xiong, J.}, \bibinfo{author}{Wang, G.}, \bibinfo{year}{2020}.
\newblock \bibinfo{title}{Universal approximation with quadratic deep networks}.
\newblock \bibinfo{journal}{Neural Networks} \bibinfo{volume}{124}, \bibinfo{pages}{383--392}.
\bibitem[{Feldmann et~al.(2021)Feldmann, Youngblood, Karpov, Gehring, Li, Stappers, Le~Gallo, Fu, Lukashchuk, Raja et~al.}]{feldmann2021parallel}
\bibinfo{author}{Feldmann, J.}, \bibinfo{author}{Youngblood, N.}, \bibinfo{author}{Karpov, M.}, \bibinfo{author}{Gehring, H.}, \bibinfo{author}{Li, X.}, \bibinfo{author}{Stappers, M.}, \bibinfo{author}{Le~Gallo, M.}, \bibinfo{author}{Fu, X.}, \bibinfo{author}{Lukashchuk, A.}, \bibinfo{author}{Raja, A.S.}, et~al., \bibinfo{year}{2021}.
\newblock \bibinfo{title}{Parallel convolutional processing using an integrated photonic tensor core}.
\newblock \bibinfo{journal}{Nature} \bibinfo{volume}{589}, \bibinfo{pages}{52--58}.
\bibitem[{Feller(1968)}]{feller1968introduction}
\bibinfo{author}{Feller, W.}, \bibinfo{year}{1968}.
\newblock \bibinfo{title}{An Introduction to Probability Theory and Its Applications, Volume I}.
\newblock \bibinfo{edition}{3} ed., \bibinfo{publisher}{John Wiley \& Sons}.
\bibitem[{Frankle and Carbin(2019)}]{frankle2019lottery}
\bibinfo{author}{Frankle, J.}, \bibinfo{author}{Carbin, M.}, \bibinfo{year}{2019}.
\newblock \bibinfo{title}{The lottery ticket hypothesis: Finding sparse, trainable neural networks}, in: \bibinfo{booktitle}{Proceedings of the International Conference on Learning Representations}.
\bibitem[{Frankle et~al.(2020)Frankle, Dziugaite, Roy and Carbin}]{frankle2020linear}
\bibinfo{author}{Frankle, J.}, \bibinfo{author}{Dziugaite, G.K.}, \bibinfo{author}{Roy, D.}, \bibinfo{author}{Carbin, M.}, \bibinfo{year}{2020}.
\newblock \bibinfo{title}{Linear mode connectivity and the lottery ticket hypothesis}, in: \bibinfo{booktitle}{International Conference on Machine Learning}, \bibinfo{organization}{PMLR}. pp. \bibinfo{pages}{3259--3269}.
\bibitem[{Glorot and Bengio(2010)}]{glorot2010understanding}
\bibinfo{author}{Glorot, X.}, \bibinfo{author}{Bengio, Y.}, \bibinfo{year}{2010}.
\newblock \bibinfo{title}{Understanding the difficulty of training deep feedforward neural networks}, in: \bibinfo{booktitle}{Proceedings of the thirteenth international conference on artificial intelligence and statistics}, \bibinfo{organization}{JMLR Workshop and Conference Proceedings}. pp. \bibinfo{pages}{249--256}.
\bibitem[{He et~al.(2015)He, Zhang, Ren and Sun}]{he2015delving}
\bibinfo{author}{He, K.}, \bibinfo{author}{Zhang, X.}, \bibinfo{author}{Ren, S.}, \bibinfo{author}{Sun, J.}, \bibinfo{year}{2015}.
\newblock \bibinfo{title}{Delving deep into rectifiers: Surpassing human-level performance on imagenet classification}, in: \bibinfo{booktitle}{Proceedings of the IEEE international conference on computer vision}, pp. \bibinfo{pages}{1026--1034}.
\bibitem[{Hopcroft et~al.(1983)Hopcroft, Ullman and Aho}]{hopcroft1983data}
\bibinfo{author}{Hopcroft, J.E.}, \bibinfo{author}{Ullman, J.D.}, \bibinfo{author}{Aho, A.V.}, \bibinfo{year}{1983}.
\newblock \bibinfo{title}{Data structures and algorithms}. volume \bibinfo{volume}{175}.
\newblock \bibinfo{publisher}{Addison-wesley Boston, MA, USA:}.
\bibitem[{Hornik et~al.(1989)Hornik, Stinchcombe and White}]{Hornik1989Multilayer}
\bibinfo{author}{Hornik, K.}, \bibinfo{author}{Stinchcombe, M.}, \bibinfo{author}{White, H.}, \bibinfo{year}{1989}.
\newblock \bibinfo{title}{Multilayer feedforward networks are universal approximators}.
\newblock \bibinfo{journal}{Neural networks} \bibinfo{volume}{2}, \bibinfo{pages}{359--366}.
\bibitem[{Huang et~al.(2009)Huang, Guestrin and Guibas}]{huang2009fourier}
\bibinfo{author}{Huang, J.}, \bibinfo{author}{Guestrin, C.}, \bibinfo{author}{Guibas, L.}, \bibinfo{year}{2009}.
\newblock \bibinfo{title}{Fourier theoretic probabilistic inference over permutations.}
\newblock \bibinfo{journal}{Journal of machine learning research} \bibinfo{volume}{10}.
\bibitem[{Hungerford(2012)}]{hungerford2012algebra}
\bibinfo{author}{Hungerford, T.W.}, \bibinfo{year}{2012}.
\newblock \bibinfo{title}{Algebra}. volume~\bibinfo{volume}{73}.
\newblock \bibinfo{publisher}{Springer Science \& Business Media}.
\bibitem[{Jordan et~al.(2023)Jordan, Sedghi, Saukh, Entezari and Neyshabur}]{jordan2023repair}
\bibinfo{author}{Jordan, K.}, \bibinfo{author}{Sedghi, H.}, \bibinfo{author}{Saukh, O.}, \bibinfo{author}{Entezari, R.}, \bibinfo{author}{Neyshabur, B.}, \bibinfo{year}{2023}.
\newblock \bibinfo{title}{{R}{E}{P}{A}{I}{R}: {R}{E}normalizing permuted activations for interpolation repair}, in: \bibinfo{booktitle}{The Eleventh International Conference On Representation Learning}.
\bibitem[{Kingma and Ba(2015)}]{kingma2015adam}
\bibinfo{author}{Kingma, D.P.}, \bibinfo{author}{Ba, J.}, \bibinfo{year}{2015}.
\newblock \bibinfo{title}{Adam: A method for stochastic optimization}, in: \bibinfo{booktitle}{Proceedings of the International Conference on Learning Representations}.
\bibitem[{Kosuge et~al.(2021a)Kosuge, Hamada and Kuroda}]{kosuge202116}
\bibinfo{author}{Kosuge, A.}, \bibinfo{author}{Hamada, M.}, \bibinfo{author}{Kuroda, T.}, \bibinfo{year}{2021}a.
\newblock \bibinfo{title}{A 16 nj/classification fpga-based wired-logic dnn accelerator using fixed-weight non-linear neural net}.
\newblock \bibinfo{journal}{IEEE Journal on Emerging and Selected Topics in Circuits and Systems} \bibinfo{volume}{11}, \bibinfo{pages}{751--761}.
\bibitem[{Kosuge et~al.(2021b)Kosuge, Hsu, Hamada and Kuroda}]{kosuge20210}
\bibinfo{author}{Kosuge, A.}, \bibinfo{author}{Hsu, Y.C.}, \bibinfo{author}{Hamada, M.}, \bibinfo{author}{Kuroda, T.}, \bibinfo{year}{2021}b.
\newblock \bibinfo{title}{A 0.61-$\mu$j/frame pipelined wired-logic dnn processor in 16-nm fpga using convolutional non-linear neural network}.
\newblock \bibinfo{journal}{IEEE Open Journal of Circuits and Systems} \bibinfo{volume}{3}, \bibinfo{pages}{4--14}.
\bibitem[{Lee et~al.(2019)Lee, Lee, Kim, Kosiorek, Choi and Teh}]{Lee2019Set}
\bibinfo{author}{Lee, J.}, \bibinfo{author}{Lee, Y.}, \bibinfo{author}{Kim, J.}, \bibinfo{author}{Kosiorek, A.}, \bibinfo{author}{Choi, S.}, \bibinfo{author}{Teh, Y.W.}, \bibinfo{year}{2019}.
\newblock \bibinfo{title}{Set transformer: A framework for attention-based permutation-invariant neural networks}, in: \bibinfo{booktitle}{Proceedings of the 36th International Conference on Machine Learning}, \bibinfo{publisher}{PMLR}. pp. \bibinfo{pages}{3744--3753}.
\bibitem[{Leshno et~al.(1993)Leshno, Lin, Pinkus and Schocken}]{Leshno1993Multilayer}
\bibinfo{author}{Leshno, M.}, \bibinfo{author}{Lin, V.Y.}, \bibinfo{author}{Pinkus, A.}, \bibinfo{author}{Schocken, S.}, \bibinfo{year}{1993}.
\newblock \bibinfo{title}{Multilayer feedforward networks with a nonpolynomial activation function can approximate any function}.
\newblock \bibinfo{journal}{Neural Networks} \bibinfo{volume}{6}, \bibinfo{pages}{861--867}.
\bibitem[{Lu et~al.(2021)Lu, Jin, Pang, Zhang and Karniadakis}]{lu2021learning}
\bibinfo{author}{Lu, L.}, \bibinfo{author}{Jin, P.}, \bibinfo{author}{Pang, G.}, \bibinfo{author}{Zhang, Z.}, \bibinfo{author}{Karniadakis, G.E.}, \bibinfo{year}{2021}.
\newblock \bibinfo{title}{Learning nonlinear operators via deeponet based on the universal approximation theorem of operators}.
\newblock \bibinfo{journal}{Nature machine intelligence} \bibinfo{volume}{3}, \bibinfo{pages}{218--229}.
\bibitem[{Lu et~al.(2017)Lu, Pu, Wang, Hu and Wang}]{Lu2017Expressive}
\bibinfo{author}{Lu, Z.}, \bibinfo{author}{Pu, H.}, \bibinfo{author}{Wang, F.}, \bibinfo{author}{Hu, Z.}, \bibinfo{author}{Wang, L.}, \bibinfo{year}{2017}.
\newblock \bibinfo{title}{The expressive power of neural networks: A view from the width}.
\bibitem[{Maltoni and Lomonaco(2019)}]{maltoni2019continuous}
\bibinfo{author}{Maltoni, D.}, \bibinfo{author}{Lomonaco, V.}, \bibinfo{year}{2019}.
\newblock \bibinfo{title}{Continuous learning in single-incremental-task scenarios}.
\newblock \bibinfo{journal}{Neural Networks} \bibinfo{volume}{116}, \bibinfo{pages}{56--73}.
\bibitem[{Maron et~al.(2019)Maron, Ben-Hamu, Shamir and Lipman}]{maron2018invariant}
\bibinfo{author}{Maron, H.}, \bibinfo{author}{Ben-Hamu, H.}, \bibinfo{author}{Shamir, N.}, \bibinfo{author}{Lipman, Y.}, \bibinfo{year}{2019}.
\newblock \bibinfo{title}{Invariant and equivariant graph networks}, in: \bibinfo{booktitle}{International Conference on Learning Representations}.
\bibitem[{Nugent(2005)}]{nugent2005physical}
\bibinfo{author}{Nugent, A.}, \bibinfo{year}{2005}.
\newblock \bibinfo{title}{Physical neural network design incorporating nanotechnology}.
\newblock \bibinfo{note}{US Patent 6,889,216}.
\bibitem[{Qiu and Suda(2020)}]{qiu2020train}
\bibinfo{author}{Qiu, Y.}, \bibinfo{author}{Suda, R.}, \bibinfo{year}{2020}.
\newblock \bibinfo{title}{Train-by-reconnect: Decoupling locations of weights from their values}.
\newblock \bibinfo{journal}{Advances in Neural Information Processing Systems} \bibinfo{volume}{33}, \bibinfo{pages}{20952--20964}.
\bibitem[{Raissi et~al.(2019)Raissi, Perdikaris and Karniadakis}]{raissi2019physics}
\bibinfo{author}{Raissi, M.}, \bibinfo{author}{Perdikaris, P.}, \bibinfo{author}{Karniadakis, G.E.}, \bibinfo{year}{2019}.
\newblock \bibinfo{title}{Physics-informed neural networks: A deep learning framework for solving forward and inverse problems involving nonlinear partial differential equations}.
\newblock \bibinfo{journal}{Journal of Computational physics} \bibinfo{volume}{378}, \bibinfo{pages}{686--707}.
\bibitem[{Rudin(1976)}]{rudin1953principles}
\bibinfo{author}{Rudin, W.}, \bibinfo{year}{1976}.
\newblock \bibinfo{title}{Principles of mathematical analysis}.
\newblock \bibinfo{edition}{3d} ed., \bibinfo{publisher}{McGraw-Hill}, \bibinfo{address}{New York}.
\bibitem[{Satorras et~al.(2021)Satorras, Hoogeboom and Welling}]{satorras2021n}
\bibinfo{author}{Satorras, V.G.}, \bibinfo{author}{Hoogeboom, E.}, \bibinfo{author}{Welling, M.}, \bibinfo{year}{2021}.
\newblock \bibinfo{title}{E (n) equivariant graph neural networks}, in: \bibinfo{booktitle}{International conference on machine learning}, \bibinfo{organization}{PMLR}. pp. \bibinfo{pages}{9323--9332}.
\bibitem[{Scabini et~al.(2022)Scabini, De~Baets and Bruno}]{scabini2022improving}
\bibinfo{author}{Scabini, L.}, \bibinfo{author}{De~Baets, B.}, \bibinfo{author}{Bruno, O.M.}, \bibinfo{year}{2022}.
\newblock \bibinfo{title}{Improving deep neural network random initialization through neuronal rewiring}.
\newblock \bibinfo{journal}{arXiv preprint arXiv:2207.08148} .
\bibitem[{Shen et~al.(2022)Shen, Yang and Zhang}]{shen2022optimal}
\bibinfo{author}{Shen, Z.}, \bibinfo{author}{Yang, H.}, \bibinfo{author}{Zhang, S.}, \bibinfo{year}{2022}.
\newblock \bibinfo{title}{Optimal approximation rate of relu networks in terms of width and depth}.
\newblock \bibinfo{journal}{Journal de Math{\'e}matiques Pures et Appliqu{\'e}es} \bibinfo{volume}{157}, \bibinfo{pages}{101--135}.
\bibitem[{Stein and Shakarchi(2009)}]{stein2009real}
\bibinfo{author}{Stein, E.M.}, \bibinfo{author}{Shakarchi, R.}, \bibinfo{year}{2009}.
\newblock \bibinfo{title}{Real analysis: measure theory, integration, and Hilbert spaces}.
\newblock \bibinfo{publisher}{Princeton University Press}.
\bibitem[{Stone(1948)}]{stone1948generalized}
\bibinfo{author}{Stone, M.H.}, \bibinfo{year}{1948}.
\newblock \bibinfo{title}{The generalized weierstrass approximation theorem}.
\newblock \bibinfo{journal}{Mathematics Magazine} \bibinfo{volume}{21}, \bibinfo{pages}{237--254}.
\bibitem[{Telgarsky(2016)}]{telgarsky2016benefits}
\bibinfo{author}{Telgarsky, M.}, \bibinfo{year}{2016}.
\newblock \bibinfo{title}{Benefits of depth in neural networks}, in: \bibinfo{booktitle}{Conference on learning theory}, \bibinfo{organization}{PMLR}. pp. \bibinfo{pages}{1517--1539}.
\bibitem[{Wang and Qu(2022)}]{wang2022approximation}
\bibinfo{author}{Wang, M.X.}, \bibinfo{author}{Qu, Y.}, \bibinfo{year}{2022}.
\newblock \bibinfo{title}{Approximation capabilities of neural networks on unbounded domains}.
\newblock \bibinfo{journal}{Neural Networks} \bibinfo{volume}{145}, \bibinfo{pages}{56--67}.
\bibitem[{Yarotsky(2017)}]{yarotsky2017error}
\bibinfo{author}{Yarotsky, D.}, \bibinfo{year}{2017}.
\newblock \bibinfo{title}{Error bounds for approximations with deep relu networks}.
\newblock \bibinfo{journal}{Neural Networks} \bibinfo{volume}{94}, \bibinfo{pages}{103--114}.
\bibitem[{Zaheer et~al.(2017)Zaheer, Kottur, Ravanbakhsh, Poczos, Salakhutdinov and Smola}]{Zaheer2017Deep}
\bibinfo{author}{Zaheer, M.}, \bibinfo{author}{Kottur, S.}, \bibinfo{author}{Ravanbakhsh, S.}, \bibinfo{author}{Poczos, B.}, \bibinfo{author}{Salakhutdinov, R.R.}, \bibinfo{author}{Smola, A.J.}, \bibinfo{year}{2017}.
\newblock \bibinfo{title}{Deep sets}, in: \bibinfo{booktitle}{Advances in Neural Information Processing Systems}.
\bibitem[{Zeng et~al.(2019)Zeng, Chen, Cui and Yu}]{zeng2019continual}
\bibinfo{author}{Zeng, G.}, \bibinfo{author}{Chen, Y.}, \bibinfo{author}{Cui, B.}, \bibinfo{author}{Yu, S.}, \bibinfo{year}{2019}.
\newblock \bibinfo{title}{Continual learning of context-dependent processing in neural networks}.
\newblock \bibinfo{journal}{Nature Machine Intelligence} \bibinfo{volume}{1}, \bibinfo{pages}{364--372}.

\end{thebibliography}

\end{document}